\documentclass[11pt]{article}
\usepackage{hyperref}

\usepackage{subcaption}
\usepackage{graphicx}
\usepackage{booktabs}
\usepackage{placeins}

\usepackage[toc,page]{appendix}
\usepackage[utf8]{inputenc}
\usepackage[T1]{fontenc}
\usepackage[normalem]{ulem}
\usepackage{amsmath, latexsym, stmaryrd, amsthm,amssymb
	,xcolor,bbm, multirow, multicol, makecell,pgfplots,bm
	,tabularx,a4wide,xurl,breakcites,enumitem,mathtools,titlesec
	,libertineRoman, authblk
	}
\usepackage[nameinlink]{cleveref} 

\hypersetup{
	colorlinks=true,
	linkcolor=blue,         
	citecolor=blue,         
	urlcolor=blue,
	pdfstartview=FitV
}

\numberwithin{equation}{section}

\newcommand{\s}[1]{\left\{ #1\right\} }

\newcommand{\tonde}[1]{\left( #1 \right)}
\newcommand{\quadre}[1]{\left[ #1 \right]}

\newcommand{\abs}[1]{\left\vert #1 \right\vert }
\newcommand{\n}[1]{\left\| #1 \right\|}
\newcommand{\nn}[1]{\left\| #1 \right\|^2}

\newcommand{\ee}{\, \wedge \,} 

\newcommand{\fre}{\ar@{-}} 

\newcommand{\w}{\omega}

\newcommand{\N}{\mathbb{N}}
\newcommand{\R}{\mathbb{R}}

\newcommand{\E}{\mathbb{E}}

\renewcommand{\S}{\mathbb S}
\renewcommand{\P}{\mathbb P}

\newcommand{\cC}{\mathcal C}

\newcommand{\di }{\, \mathrm{d}}

\DeclareMathOperator{\Var}{Var}

\DeclareMathOperator{\GOI}{GOI}

\def\be{\begin{eqnarray}}
	\def\ee{\end{eqnarray}}
\def\ben{\begin{eqnarray*}}
	\def\een{\end{eqnarray*}}

\newtheorem{theorem}{Theorem}[section] 
\newtheorem{lemma}[theorem]{Lemma}     
\AddToHook{env/lemma/begin}{\crefalias{theorem}{lemma}}
\newtheorem{proposition}[theorem]{Proposition}
\AddToHook{env/proposition/begin}{\crefalias{theorem}{proposition}}

\AddToHook{env/corollary/begin}{\crefalias{theorem}{corollary}}

\theoremstyle{definition}
\newtheorem{definition}[theorem]{Definition}
\AddToHook{env/definition/begin}{\crefalias{theorem}{definition}}

\AddToHook{env/example/begin}{\crefalias{theorem}{example}}
\newtheorem{remark}[theorem]{Remark}
\AddToHook{env/remark/begin}{\crefalias{theorem}{remark}}

	\usepackage{tikz}
	\usepackage{pgfplots}
	\usetikzlibrary{external}
	\tikzset{every picture/.append style={font=\rmfamily}}
	\definecolor{darkgray176}{RGB}{176,176,176}
	\definecolor{darkorange25512714}{RGB}{255,127,14}
	\definecolor{forestgreen4416044}{RGB}{44,160,44}
	\definecolor{lightgray204}{RGB}{204,204,204}
	\definecolor{steelblue31119180}{RGB}{31,119,180}
\usepackage{xcolor}
\usepackage{booktabs}

\usepackage{array}
\newcolumntype{C}{>{\centering\arraybackslash}p{3cm}}

\titleformat{\subsection}[runin]
{\normalfont\large\bfseries}{\thesubsection}{1em}{}

\title{Critical Points of Random Neural Networks}

\author{Simmaco Di Lillo}
\affil{RoMaDS - Department of Mathematics, University of Rome Tor Vergata, 	Rome, Italy \newline
	{\tt{dilillo@mat.uniroma2.it}}}

\date{}                     

\newcommand{\doi}[1]{\href{https://doi.org/#1}{doi: #1}}

\begin{document}

	\maketitle              

	\begin{abstract}This work investigates the expected number of critical points of random neural networks with different activation functions as the depth increases in the infinite-width limit. Under suitable regularity conditions, we derive precise asymptotic formulas for the expected number of critical points of fixed index and  those exceeding a given threshold. Our analysis reveals three distinct regimes depending on the value of the first derivative of the covariance evaluated at $1$: the expected number of critical points may converge, grow polynomially, or grow exponentially with depth. Our theoretical predictions are supported by numerical experiments. Moreover, we provide numerical evidence suggesting that, when the regularity conditions are not satisfied, the number of critical points increases with map resolution, suggesting a potential divergence.
	\noindent \textbf{Keywords}: Isotropic Gaussian Random Field, Critical Points, Deep neural networks, Kac-Rice formula, Gaussian Orthogonally Invariant Matrix.\\
	\noindent \textbf{MSCcodes}:  60G60, 62B10, 62M45	.\\

	\end{abstract}
	\section{Introduction}Deep architectures have consistently achieved state-of-the-art performance in applications such as
	image classification \cite{krizhevsky2017imagenet}, speech recognition \cite{hinton2012deep}, natural language processing \cite{vaswani2017attention}, biomedical imaging \cite{schenone2020radiomics}, clinical decision support \cite{campi2021antibiotic} and protein structure prediction \cite{jumper2021highly}.

	Given these practical successes, there has been growing interest in developing a rigorous theoretical understanding of neural networks, with contributions emerging from approximation theory \cite{Bartlett_Montanari_Rakhlin_2021}, statistical approaches \cite{MR4134774}, probabilistic ones based on Gaussian processes \cite{10.5555/1162254} and based on reproducing kernel theory \cite{aronszajn50reproducing}.  Alongside these developments, research has increasingly focused on  the geometric aspects of neural networks, aiming to capture the complexity of different architectures through their structural properties~\cite{NIPS2014_109d2dd3,pmlr-v97-hanin19a,GOUJON2024115667,doi:10.1137/22M147517X,6697897,10.5555/3016100.3016187}. At the same time, building on Neal’s seminal paper~\cite{neal1996priors}, several studies have investigated wide neural networks at initialization through the lens of Gaussian processes.  It is now well established that, under random initialization and a suitable scaling of the variance of weights and biases, fully connected feedforward neural networks converge to Gaussian processes as the width of each layer tends to infinity \cite{neal2012bayesian,NIPS1996_ae5e3ce4,NIPS2016_abea47ba,lee2018deep,g.2018gaussian,hanin2,cammarota_marinucci_salvi_vigogna_2023,balas24,favaro2025quantitative,NEURIPS2018_a1afc58c}.
	In particular, by normalizing the inputs to lie on the unit sphere, neural networks can be interpreted as random fields on the hypersphere, allowing the use of harmonic analysis and stochastic geometry to describe the underlying structure of the models. Recent studies have used this approach to analyze the angular power spectrum of such fields \cite{nostro}, leading to a classification of activation functions into three different regimes: low-disorder, sparse and high-disorder regimes. This classification reflects how the complexity of the network varies with depth and has been used to explain, among other things, the observed sparsity of  networks and their relative robustness to overfitting.
	Parallel to the analysis of the spectrum, a second line of inquiry has focused on the geometry of level and excursion sets of random neural networks.  In particular, classical results from the theory of Gaussian fields, such as the Kac-Rice formula \cite{book:145181,book:73784}, have been used to calculate expectations of topological invariants (e.g., Euler characteristic, Lipschitz-Killing curvatures) and to study the distribution of critical points. Such tools have been extensively applied in the context of statistical cosmology \cite{marinucci2011random}, and their adaptation to deep learning settings provides new insights into the geometric complexity of neural representations \cite{nostro2}.
	Building on these advances, the present work aims to investigate the distribution of critical points of random neural networks at initialization. More precisely, we consider neural networks with input in the $d$ dimensional sphere $\S^d$ and with real-valued output, and we study the expected number of critical points of the limiting Gaussian field in the infinite-width regime.

	Our analysis extends previous works in two main directions.  First, we extend the notion of \emph{Covariance Regularity Index} ($\mathrm{CRI}$), originally introduced in \cite{nostro2} for $\mathrm{CRI}$ values below 2, which captures the smoothness of the covariance function near the diagonal. As in~\cite{nostro2}, this index determines whether classical Kac--Rice techniques can be applied and reveals a sharp dichotomy: when the $\mathrm{CRI}$ exceeds 2, the field is almost surely $C^2$ (see~\Cref{prop::1}), and the expected number of critical points of the given index can be computed explicitly using the machinery developed in \cite{book:145181,book:73784} and in particular using a  general Kac--Rice framework for smooth isotropic fields on the sphere~\cite[Theorem 4.4]{cheng2018expected}. In contrast, for irregular fields such as those associated with ReLU activation, where the $\mathrm{CRI}$ is strictly less than 2, the assumptions required by the classical Kac-Rice framework fail: the field is not smooth enough to guarantee nondegenerate joint laws of the gradient and Hessian. As a consequence, standard formulas for the expected number of critical points do not apply, and numerical evidence points to a possible divergence as the resolution increases.  Second, we identify a simple spectral parameter that governs the asymptotic behavior of the expected number of critical points as the depth increases (see~\Cref{th2}). Depending on whether  the derivative of the covariance kernel at the origin  is less than, equal to, or greater than one, the number of critical points follows qualitatively distinct scaling behaviors, corresponding, respectively, to vanishing, bounded, or exponentially growing complexity (see~\Cref{tab:trichotomy} for an informal version of~\Cref{th2}). This classification mirrors the one found in \cite{nostro} from a spectral perspective, but it concerns a geometric quantity. We further extend our results to count critical points above a given  threshold, and we derive asymptotic formulas for the expected number of such points (see~\Cref{th3}).
	\newpage
	\begin{table*}[!htb]
		\centering
		\caption{ Informal trichotomy for the expected number of critical points with fixed index  $\mathbb{E}[\cC_i(T_L)]$ as depth $L$ grows. For more examples of activation function, see~\cite[Table~1]{nostro}. $\kappa'(1)$ refers to the derivative of the covariance kernel $\kappa(u)$ at $u=1$.}
		\label{tab:trichotomy}
		\begin{tabularx}{\textwidth}{@{} X X X @{}}
			\toprule
			Regime    & Growth of $\mathbb{E}[\cC_i(T_L)]$ & Example activation \\
			\midrule
			Low-disorder: $\kappa'(1) < 1$  & Bounded in $L$ (order $1$)       & Gaussian, $a^2=1$ \\
			Sparse: $\kappa'(1) = 1$     & Polynomial, order $L^{d/2}$      & Gaussian, $a^2=1+\sqrt{2}$ \\
			High-disorder: $\kappa'(1) > 1$& Exponential in $L$               & Gaussian, $a^2=9$ \\
			\bottomrule
		\end{tabularx}
	\end{table*}
	The rest of the paper is organized as follows. In~\Cref{back} we provide some background information on spherical random fields, random neural networks, and GOI matrices. In~\Cref{main} we state our main results. The complete proofs of the two theorems, together with some technical lemmas, are collected in~\Cref{proofs}. In~\Cref{numerical_evidence}, our theoretical findings are supported by extensive Monte Carlo simulations using the HEALPix package \cite{gorski2005healpix}, which confirm the predicted scaling behaviors and highlight the role of the activation function in shaping the geometric structure of the network. The empirical evidence suggests that, in irregular cases such as ReLU, for instance, the number of critical points diverges, consistent with a set of critical points with non-integer Hausdorff dimension. In Appendix~\ref{appA} we prove that a value of $\mathrm{CRI} > 2$ ensures that the fields are smooth,
	namely that their second derivatives are Hölder continuous.
	In Appendix~\ref{riemann}, we recall how the first derivative and the Hessian can be defined on the sphere
	and we generalize a classical result on the covariances of the field and of its derivatives
	(to the case of fields with $\mathrm{CRI} > 2$).
	Appendix~\ref{appC} collects some technical details that are useful in the proofs presented in Appendix~\ref{appA}.

	\section{Background and notation}\label{back}
	In this section, we provide the necessary background and introduce the notation used throughout the paper.

	\subsection{Spherical random fields}\label{smot}

	Let $(\Omega, \mathcal F,   \P)$ be a fixed probability space and let
	$T : \S^d \times \Omega \to \R $ be a finite-variance, zero-mean, isotropic, Gaussian random field on the $d$-dimensional unit sphere $\S^d$.
	This means that the following properties hold:
	\begin{itemize}
		\item Measurability. The field $T$ is $(\mathcal B (\S^d) \otimes \mathcal F)$-measurable, where $\mathcal B (\S^d)$ denotes the Borel $\sigma$-algebra on $\S^d$.
		\item Isotropy.  For every $g\in \mathrm{SO}(d+1)$ (the group of rotations in $\R^{d+1}$), for every $ k \geq 1 $ and every $x_1, \dotsc, x_k \in \S^d$, the two random vectors $(T(x_1), \dotsc, T(x_k) )  $ and $(T(gx_1), \dotsc, T(g x_k) ) $ have the same distribution.
		\item Gaussianity. For every $k\geq1$ and every $x_1, \dotsc, x_k \in \S^d$, the random vector $(T(x_1), \dotsc, T(x_k) )$ is Gaussian.
		\item Zero mean. 	For all $x\in \S^d$, $\E[T(x)] =0 $.
		\item Finite variance. For all $x\in \S^d$, $\Var(T(x))<\infty $.
	\end{itemize}
	Given any spherical random field $T$ as above, we may introduce the covariance function.
	\begin{align*}
		K & : \mathbb{S}^d \times \S^d \to \mathbb{R} ,  \qquad K(x,y):= \mathbb{E}[T(x) T(y)].
	\end{align*}
	By isotropy, the covariance can be expressed as a function of the angle between two points, i.e., there exists $\kappa : [-1, 1] \to \mathbb{R}$ such that
	$$ K(x,y)  = \kappa(\langle x , y \rangle) , $$
	where $\langle \cdot, \cdot \rangle$ denotes the scalar product in   $\R^{d+1}$. See~[Remark 5.15 (2)]\cite{marinucci2011random} for more details.

	\subsection{Random Neural Networks}
	For any $L$, a random neural network of depth $L$, widths $n_1, \dots, n_L$, activation function $\sigma:\, \R \to \R$,  and bias variance equal to $\Lambda_b \in [0,1)$, is the random field $T_L:\, \S^{d} \to \R^{n_{L+1}}$ defined recursively as follows:
	$$ T_s(x):=\begin{cases}W^{(0)}x + b^{(1)}, & \text{ if } s=0,\\
		W^{(s)}\sigma(T_{s-1}(x)) + b^{(s+1)},& \text{ if } s =1, \dots, L ,
	\end{cases}, \qquad x\in \S^{d}
	$$
	where $\sigma$ is  applied component-wise and this    $(W^{(s)}, b^{(s)})$  are  independent random variables. This variablesdefined as follows.
	\begin{itemize}
		\item For each $s$, $W^{(s)} \in \R^{n_{s+1} \times n_s}$  has i.i.d. components and  $n_0 = d+1$, where:
		\begin{align*} W^{(0)}_{ij} \sim \mathcal N(0,1-\Lambda_{b}) \quad & s =0,\\
			W^{(s)}_{ij} \sim \mathcal N(0, \Lambda_W n_s^{-1/2}) \quad & s>0
		\end{align*}
		with
		$$\Lambda_W:=\frac{1-\Lambda_b}{ \E[\sigma(Z)^2]} \qquad Z\sim \mathcal N(0,1). $$
		\item  For each $s$,  $b^{(s)} \in \R^{n_s}$  with i.i.d. components such that
		$$b^{(s)}_i \sim \mathcal N(0, \Lambda_b) . $$
	\end{itemize}

	It is well known (see~\cite{pmlr-v97-hanin19a,hanin2,bastieri2024} and references therein) that
	the random neural network $T_L$ converges weakly to an isotropic, zero-mean,  Gaussian random field with $n_{L+1}$ i.i.d. components
	as the widths $n_1,\dots,n_L$ go to infinity.
	In particular, owing to~\cite{nostro}, if $K_L$ denotes the covariance function of one component of the limiting process, we have
	\begin{equation} \label{eq:KL0}
		\begin{aligned}
			& K_L(x,y) = \kappa_L(\langle x, y \rangle):= \underbrace{\kappa\circ\cdots \circ \kappa}_{L \text{ times}}(\langle x, y \rangle)
		\end{aligned}
	\end{equation}
	with
	\begin{equation}\label{eq:KL1}
		\kappa(u) := \Lambda_W \E\quadre{ \sigma(Z_1)\sigma(uZ_1 + \sqrt{1-u^2} Z_2)} + \Lambda_b
	\end{equation}
	where $Z_1, Z_2$ are independent standard Gaussian random variable. More precisly
	\begin{equation}\label{vvalue1} \kappa(1) =\Lambda_W  \E_{Z\sim \mathcal N(0,1)} \E[ \sigma(Z)^2] + \Lambda_b = 1 .
	\end{equation}
	We note that the choice of $\Lambda_b$ and $\Lambda_W$ ensures that the field has unit variance.

	With a slight abuse of terminology, the limit
	Gaussian fields will also be called random neural networks.

	\subsection{Gaussian Orthogonally Invariant Matrix }
	In this section we follows the notation in~\cite{GOI}. Let $d$ be a natural number and let $c\in \R$ such that $1+dc\ge 0$. A real symmetric random
	matrix $M\in \R^{d\times d}$  is said to be Gaussian Orthogonally Invariant (GOI) with covariance parameter $c$ (and denoted by $\GOI(c)$), if its entries $M_{ij}$ are zero-mean  Gaussian random variables and
	$$ \E[M_{ij} M_{hk}] =
	\frac 1 2 (\delta_{ih}\delta_{jk}  + \delta_{ik}\delta_{jh})  + c \delta_{ij}\delta_{hk}  . $$
	The condition $1+dc> 0$ implies that the random matricies is nondegenerate~\cite[Lemma 2.1]{cheng2018expected}. In~\cite[Lemma 2.2]{cheng2018expected} the authors derived the joint
	density of the ordered eigenvalues of a $\GOI(c)$ matrix, given by
	\begin{equation}\label{goi_def}
		\begin{aligned}
			f_c(\lambda_1,\dots, \lambda_d):=\\
			\frac {\prod_{1\leq h<k\leq d} |\lambda_h -\lambda_k| }{ K_d \sqrt{1+dc}}\exp\tonde{ - \frac {\sum_{j=1}^d \lambda_j^2} 2 + \frac c{2(1+dc)}\tonde{\sum_{j=1}^d \lambda_i}^2}  \mathbb I_{\lambda_1\leq \dots \leq \lambda_d },
		\end{aligned}
	\end{equation}
	where
	$$ K_d:=  2^{d/2}\prod_{j=1}^d \Gamma(j/2) $$
	is a normalization constant, $\Gamma$ denotes the Euler Gamma function and
	$$
	\mathbb I_{\lambda_1\leq \dots \leq \lambda_d }:=\begin{cases} 1 & \text{ if }  {\lambda_1\leq \dots \leq \lambda_d },\\
		0 & \text{ otherwise } .
	\end{cases}$$
	%
	%
	\section{Main results}\label{main}
	Let us start by extending the definition of Covariance
	Regularity Index given in~\cite{nostro2}.
	\begin{definition} Let  $T$ be a random field and let $\kappa$ be its covariance
		function.  We assume that $\kappa\in C^\infty((-1,1))$. We say that $T$ has Covariance Regularity Index ($\mathrm{CRI}$) equal to $\beta>0$ if  either
		\begin{itemize}
			\item $\beta\not \in \N$, $\kappa\in C^{\lceil \beta \rceil} ((-1,1))$ and $\kappa$ admits the following expansions around $\pm 1$ as $t\to 0^+$
			\begin{align*} \kappa(1-t) = p_1(t) + c_1t^\beta + o(t^\beta) , \\
				\kappa (-1 + t) = p_{-1}(t) + c_{-1}t^\beta  + o(t^\beta) ,
			\end{align*}
			where $\lceil \beta \rceil$ denotes the  ceiling of $\beta$,  $p_1$, $p_{-1}$ are polynomials and $c_1$, $c_{-1}$ are constants. We also assume that  the first $\lfloor \beta \rfloor$ derivatives of $\kappa$
			admit similar expansions obtained by differentiating the above ones.
		\end{itemize}
	\end{definition}

	\begin{remark}
		The assumption $\kappa\in C^\infty((-1,1))$ is fulfilled if $\kappa\in C^0([-1,1])$ and satisfies~\eqref{eq:KL1} with $\sigma \in L^2(\R, e^{-1/2x^2})$. Indeed, using the completeness of Hermite polynomials,
		$$ \sigma(Z) = \sum_{q=0}^\infty \frac{J_q(\sigma)}{q!} H_q(Z) $$
		and using the Diagramma formula,
		\begin{equation}\label{anal}\kappa(u)  =\Lambda_W\E[ \sigma(Z_1) \sigma(uZ_1 + \sqrt{1-u^2}Z_2)] + \Lambda_b =  \Lambda_W\sum_{q=0}^\infty \frac{J_q(\sigma)^2}{q!} u^q + \Lambda_b.
		\end{equation}
		Since $\kappa$ is continuous,
		\begin{equation}\label{value1} 1=\kappa(1) =\lim_{u\to 1^-}   \Lambda_W\sum_{q=0}^\infty \frac{J_q(\sigma)^2}{q!} u^q  + \Lambda_b =\Lambda_W \sum_{q=0}^\infty \frac{J_q(\sigma)^2}{q!} + \Lambda_b
		\end{equation}
		where the last inequality follows from Beppo--Levi Theorem.  Hence $\kappa$ is a power series with radius of convergence $\geq 1$ and hence it is analytic in $(-1,1)$.
	\end{remark}
	In Appendix~\ref{appA} we extend  the computation of~\cite[Proposition 3.14]{nostro2}  to the second derivatives and prove the following results.
	\begin{proposition} \label{prop::1}Let $T : \S^d \to \R$  be an isotropic Gaussian random field with zero mean and unit	variance. If $T$ has $\mathrm{CRI}>2$, then $T \in C^2(\S^d)$ almost surely and moreover the second derivatives are H\"{o}lder continuous.
	\end{proposition}

	For a given matrix $A$, we denote by $\mbox{index}(A)$ the number of its negative eigenvalues. Let $T:\, \S^d \to \R$ be a random field. We define
	$$ \mathcal C_i(T):= \big| \big\{ x\in \S^{d} \, | \,  \nabla T(x) =0, \, \mbox{index}(\nabla^2T(x)) = i \big\}
	\big| $$
	as  the number of critical points of $T$ with  signature $i$, and
	$$ \mathcal C_i(T,u):= \big| \big\{ x\in \S^{d} \, | \, T(x) \geq u, \,  \nabla T(x) =0, \, \mbox{index}(\nabla^2T(x)) = i \big\}
	\big| $$
	as the  the number of critical points of $T$ with signature $i$ where the field exceeds the threshold $u$.

	Our first result provides an asymptotic formula for the expected number of critical points of a given index.

	\begin{theorem}\label{th2} For any $L$, let $T_L$ be one component of the infinite-width limit, i.e.,  $T_L$ is a random field on $\S^d$ with zero mean and covariance function $K_L$ as in~\eqref{eq:KL0}. If the $\mathrm{CRI}$ of $T_1$ is greater than $2$ then
		\begin{align*}
			& \E[C_i(T_L) ] = B_i(\kappa'(1),\kappa''(1)) +o(1) & \quad \text{ if } \kappa'(1)<1,\\
			&\E[C_i(T_L)] =  L^{d/2}\tonde{\frac{A_i}{\eta^{d/2}}+ o(1)} &\quad \text{ if } \kappa'(1)=1, \\
			& \E[C_i(T_L)] = \kappa'(1)^{Ld/2}\tonde{\frac{A_i}{(\eta(\kappa'(1) -1))^{d/2}}+ o(1)} &\quad \text{ if } \kappa'(1)>1
		\end{align*}
		where $\eta= \kappa'(1)/\kappa''(1)$, $A_i$ is an absolute positive constant (cf.,~\eqref{Ai}) 	and $B_i$ is a positive constant that depends only on $i, d$, $\kappa'(1)$ and $\kappa''(1)$ (cf.,~\eqref{Bi}).
	\end{theorem}

	Our second result concerns the expected number of critical points of a given index that exceed a fixed threshold $u\in \R$.

	\begin{theorem}\label{th3}Let $(T_L \; | \; L\geq 1)$ as in~\Cref{th2}. Then
		\begin{align*}
			&\E[\cC_i(T_L,u)] = B_i(\kappa'(1),\kappa''(1)) (1- \Phi(u)) + o(1)& \text{ if } \kappa'(1) < 1,\\
			&\E[\cC_i(T_L,u)] = L^{d/2} \tonde{\frac{A_i}{\eta^{d/2}} (1- \Phi(u)) + o(1)} & \text{ if } \kappa'(1) = 1 ,\\
			& \E[\cC_i(T_L,u)] = \kappa'(1)^{Ld/2}(D_i(\kappa'(1), \kappa''(1),u) + o(1)) & \text{ if } \kappa'(1)> 1
		\end{align*}
		where $\Phi(u) $ is the standard Gaussian cumulative distribution function;  $A_i, B_i$ are as in~\Cref{th2}  and $D_i$ is another positive constant  that depends  on $i, d, u$, $\kappa'(1)$ and $\kappa''(1)$ (cf.,~\eqref{Di}).
	\end{theorem}

	\begin{remark}\label{rem}
		Using the computation in~\cite{bietti2021deep}, it can be shown that the $\mathrm{CRI}$  of a random neural network with activation function Rectified Linear Unit (ReLU), i.e., $\sigma(x) = \max(0,x)$, is $3/2$. Consequently, the previous theoretical results are not applicable in this case. We therefore  compute  the number of critical points numerically. Since this number increases as the map resolution increases (see~\Cref{fig::relu} and \Cref{tab::ReLu}), it  is reasonable to conjecture that a ReLU network may exhibit infinitely many critical points. A possible line of investigation could  involve adapting the argument used in the proof of Theorem 3.9 in~\cite{nostro2} to show that the set of critical points may have non-integer Hausdorff dimension. We leave a rigorous analysis of this for future work.
	\end{remark}

	\section{Proofs}\label{proofs}
	In this section, we collect the proofs of the two theorems. This section is structured in this way. We start with the notation part, after with the proof~\Cref{th2} (cf.~\Cref{proof_th1}). The proof of this theorem uses two technical lemmas, proved in~\Cref{technical}. In~\Cref{proof_th2} we prove the~\Cref{th3}.
	\subsection{Notations}
	To simplify the proofs, we make use of the following notation:
	\begin{itemize}
		\item  $\boldsymbol{\lambda}$ is a vector in $\R^d$  with $\lambda_1, \dotsc, \lambda_d$ as components. Analogously for other bold  letters.
		\item $\boldsymbol{1}$ is the vector in $\R^d$ with all components equal to $1$.
		\item If $\boldsymbol{v}, \boldsymbol{w}\in \R^d$, $\boldsymbol{v}\boldsymbol{w}^\top$ denotes the matrix $M \in \R^{d\times d}$ with entries $M_{ij} = v_i w_j $. Instead, as customary, we use $\boldsymbol v^\top \boldsymbol w$ to denote the standard inner product and $\|\boldsymbol v\|$ for the Euclidean norm.
		\item  If $\xi, x \in\R$ and $\boldsymbol{\lambda}\in \R^d$, we denote by $\boldsymbol{\lambda}_{\xi,x}$
		the vector whose components are given by $\lambda_i - \frac{\xi x } {\sqrt 2 }$ for $i=1, \dotsc,d $.
		\item For a function $g:\, \R^d \to \R$, the symbol $\E_{\GOI(c)}^d[g(\boldsymbol{\lambda})]$ is just a way to write
		\begin{equation}\label{int_GOI} \int_{\R^d} g(\boldsymbol{\lambda}) f_c(\boldsymbol{\lambda}) \di \boldsymbol{\lambda}
		\end{equation}
		where $f_c$ is the density of ordered eigenvaueles of a $\GOI(c)$ matrix given by~\eqref{GOIc}.
		\item Let $$\mathcal O:= \s{ \bold x\in \R^d \; | \; x_i  \leq x_{i+1} \; i=1, \dots, d-1 }$$
		be the subset of $\R^d$ consisting of all vectors with non-decreasing components and
		for any $i=0, \dots, d$, let   $\mathcal O_i$ be the subset of $\R^d$ given by
		\begin{align*} &\mathcal O_0:= \s{ \bold x\in \R^d \; | \; 0< x_{1}} , \\
			& \mathcal O_i:= \s{ \bold x\in \R^d \; | \; x_i<0< x_{i+1}} , \qquad i=1, \dots, d-1,\\
			&		 \mathcal O_d:= \s{ \bold x\in \R^d \; | \; x_d<0} .
		\end{align*}
		\item For any set $A\subseteq \R^d$ let $\mathbb I_{A} $ denote its indicator function, i.e.
		$$ \mathbb I_{A}(\boldsymbol{x}):= \begin{cases} 1 & \text{ if } \boldsymbol{x}\in A\\
			0 &\text{ otherwise}
		\end{cases} . $$

		\item Let $\Delta:\, \R^d\to \R$ be the function given by
		$$ \Delta(\bold x) := \prod_{1\leq i < j\leq d} |x_i - x_j| . $$
		\item  Let $ \Pi:\, \R^d\to \R$ be the function given by
		$$ \Pi(\bold x):= \prod_{i=1}^d  |x_i| . $$
	\end{itemize}

	We note that using the previous notations, one can write the density $f_c$  given in~\eqref{goi_def} in a compact way:
	\begin{equation}
		\label{dens} f_c(\boldsymbol \lambda) =\frac{1}{K_d\sqrt{1+dc}} \exp\tonde{ - \frac{\| \boldsymbol{\lambda}\|^2} 2 + \frac{c (\bold 1^\top \boldsymbol \lambda)^2}{2(1+dc)}}   \Delta(\boldsymbol \lambda ) \mathbb I_{\mathcal O}(\boldsymbol{\lambda}) .
	\end{equation}
	\subsection{$T_L$ satisfies the hypothesis of~\cite[Theorem 4.4]{cheng2018expected}}\label{verifica}
	In this section, we prove that, for every fixed $L$, the random field $T_L$ satisfies the assumptions of \cite[Theorem 4.4]{cheng2018expected}.
	For completeness, we recall the statement of the theorem.
	\begin{proposition}[Theorem 4.4~\cite{cheng2018expected}]\label{cheng2}
		Let $\s{X(t) \; | \; t\in \S^N}$ be a centered, unit-variance, smooth isotropic Gaussian random field  satisfying  $k^2 - \eta^2< (N+2)/2$
		Then for $i= 0,\dotsc, N$
		$$ \E[\mu_i(X)] = \frac{1}{\pi^{N/2} \eta^N} \E_{\mathrm{GOI}((1+\eta^2)/2)}^N \Bigg[ \prod_{j=1}^N |\lambda_j| \mathbb I_{\lambda_i < 0 < \lambda_{i+1}}\Bigg]$$
		and
		$$ \E[\mu_i(X,u)]=\frac{1}{\pi^{N/2} \eta^N} \int_u^\infty \phi(x)\E_{\mathrm{GOI}(( 1 + \eta^2 - k^2)/2)}^N \Bigg[ \prod_{j=1}^N | \lambda_j -kx/\sqrt2| \mathbb I_{\lambda_i < kx/\sqrt 2 < \lambda_{i+1}} \Bigg]$$
		where $\phi(x)$  is the density of standard Gaussian variable, $k:= C'(1)/\sqrt{C''(1)}$, $\eta:= \sqrt{C'(1)/C''(1)}$, $C$ is the covariance function,  $\E^N_{\mathrm{GOI}(c)}$  is defined in~\eqref{int_GOI},  $\mu_i(X)$ and $\mu_i(X, u)$ are, respectively,  $\mathcal{C}_i(X)$ and $\mathcal C_i(X,u)$ normalized by the surface area of the sphere $\w_d$ and  $\lambda_0$  and $\lambda_{N+1}$ are regarded respectively as $-\infty$ and $\infty$ for consistency.
	\end{proposition}
	By standard results on the infinite-width limit of random neural networks (see, e.g., \cite{pmlr-v97-hanin19a,hanin2,bastieri2024}), the field $T_L$ is centered.
	Moreover, by \eqref{eq:KL0} its covariance has the form
	\[
	\E\!\big[T_L(x)\,T_L(y)\big]=\kappa_L(\langle x,y\rangle),
	\]
	and since $\kappa(1)=1$ (cf.~\eqref{vvalue1}), the recursion defining $\kappa_L$ preserves the unit variance, so that $\kappa_L(1)=1$ as well.
	\smallskip

	\emph{Smoothness.}
	In \cite{cheng2018expected}, “smooth” random fields satisfy the modulus of continuity condition (see \cite[Eq.~(11.3.1)]{book:73784}): there exist constants $K>0$ and $\alpha>0$ such that, uniformly in $t$,
	\begin{equation}\label{adler}
		\max_{i,j}
		\E\quadre{ |E_{ij}T_L(t) - E_{ij}T_L(s)|^2} \leq K\,\bigl|\ln|t-s|\bigr|^{-(1+\alpha)} ,
	\end{equation}
	where $(E_i)_{i=1}^d$ is an orthonormal frame on $\S^d$;  see Appendix~\ref{riemann} for details on the frame.
	In Appendix~\ref{A4} we show that, if the Covariance Regularity Index $\mathrm{CRI}>2$, then \eqref{adler} holds for $T_L$.

	\smallskip
	\emph{Spectral inequality.}
	It remains to verify the inequality $k^2-\eta^2<(N+2)/2$.
	In our notation this is equivalent to
	\begin{equation}\label{deg_cond}
		\gamma_L:= \frac{\kappa_L'(1)^{2}-\kappa_L'(1)}{\kappa_L''(1)}
		< \frac{d+2}{2}\, .
	\end{equation}
	From \cite{nostro} we have, for every $L\in\mathbb N$,
	\begin{subequations}\label{as_kern}
		\begin{align}
			\label{as_kern1}
			\kappa_L'(1) &= \bigl(\kappa'(1)\bigr)^{L},\\
			\label{as_kern2}
			\kappa_L''(1) &=
			\begin{cases}
				L\,\kappa''(1), & \text{if }\kappa'(1)=1,\\[4pt]
				\kappa''(1)\,\bigl(\kappa'(1)\bigr)^{L-1}\,\dfrac{(\kappa'(1))^{L}-1}{\kappa'(1)-1},
				& \text{if }\kappa'(1)\neq 1.
			\end{cases}
		\end{align}
	\end{subequations}
	Hence
	\[
	\gamma_L =
	\begin{cases}
		0, & \text{if }\kappa'(1)=1,\\[4pt]
		\dfrac{\kappa'(1)\,(\kappa'(1)-1)}{\kappa''(1)}, & \text{if }\kappa'(1)\neq 1.
	\end{cases}
	\]
	In particular, if $\kappa'(1)\le 1$ then either $\gamma_L=0$ (when $\kappa'(1)=1$) or $\gamma_L\le 0$ (when $0\le\kappa'(1)<1$), and \eqref{deg_cond} follows.
	(Recall also that $\kappa_L''(1)=\Var\!\big(E_{ij}T_L(t)\big)\ge 0$ for any $i\neq j$ by~\Cref{lemma_gen}.)

	For the high-disorder regime $\kappa'(1)>1$, let $X$ be the $\mathbb N\cup\{0\}$-valued random variable
	\[
	\P(X=q)=\frac{J_q(\sigma)^2}{q!},\qquad q=0,1,2,\dots,
	\]
	where the coefficients $J_q(\sigma)$ are as in \eqref{anal}.
	Since $\sum_{q\ge 0} J_q(\sigma)^2/q!=\kappa(1)=1$ (cf.~\eqref{anal}), this is a probability distribution.
	Because $\kappa\in C^2([-1,1])$, monotone convergence yields
	\[
	\kappa'(1) = \sum_{q\ge 1} \frac{J_q(\sigma)^2}{q!}\,q = \E[X],
	\qquad
	\kappa''(1) = \sum_{q\ge 2} \frac{J_q(\sigma)^2}{q!}\,q(q-1) = \E\!\big[X(X-1)\big].
	\]
	By convexity of $x\mapsto x(x-1)$ and Jensen’s inequality,
	\[
	\kappa''(1)=\E\big[X(X-1)\big]\geq \E[X]\big(\E[X]-1\big)
	=\kappa'(1)\big(\kappa'(1)-1\big).
	\]
	Therefore, for $\kappa'(1)>1$,
	\[
	\gamma_L =\frac{\kappa'(1)\,(\kappa'(1)-1)}{\kappa''(1)}\leq 1 .
	\]
	Since $d\ge 2$ in our setting, we have $1\le (d+2)/2$, and thus $\gamma_L\le 1 < (d+2)/2$, which proves \eqref{deg_cond} also in the high-disorder case.
	\color{black}

	\subsection{Proof  of~\Cref{th2}}\label{proof_th1}
	Using the results in the previous section, and recalling that the volume of $\S^d$ is given by
	$$ \omega_d = \frac{2 \pi^{(d+1)/2}}{\Gamma((d+1)/2)}$$
	we obtain
	\begin{equation}\label{crit2}
		\E[\mathcal C_i(T_L)] = \E[\mu_i(T_L)]\omega_d = \frac{ 2\sqrt{\pi}}{\Gamma\tonde{\frac{d+1} 2 }\eta_L^{d/2}} \E^d_{\GOI\tonde{\frac{1+\eta_L}2}}
		\left[  \Pi(\boldsymbol{\lambda}) \mathbb I_{\mathcal O_i}(\boldsymbol\lambda) \right]
	\end{equation}
	where $\eta_L:= \kappa_L'(1) (\kappa_L''(1))^{-1}$.  Using~\eqref{as_kern1} and~\eqref{as_kern2} we obtain
	\begin{align*}
		\eta_L &= \eta_1 (1-\kappa'(1)) + o(1) \quad & \text{ if } \kappa'(1)<1 ,\\
		\eta_L & = L^{-1}  \eta_1  & \text{ if } \kappa'(1) = 1,\\
		\eta_L &= \kappa'(1)^{-L} \tonde{ \eta_1(\kappa'(1) -1)+ o(1)} \quad & \text{ if } \kappa'(1)>1.
	\end{align*}
	Now, if $\kappa'(1)\geq 1$ then $\eta_L\to 0 $ so
	$$\lim_{L\to \infty }\E^d_{\GOI\tonde{\frac{1+\eta_L}2}}\left[  \Pi(\boldsymbol{\lambda}) \mathbb I_{\mathcal O_i}(\boldsymbol\lambda) \right]=\E^d_{\GOI\tonde{\frac 1 2}}\left[  \Pi(\boldsymbol{\lambda}) \mathbb I_{\mathcal O_i}(\boldsymbol\lambda) \right]  .  $$
	This is guaranteed from~\Cref{goi_stima_1/2}, above. Putting
	\begin{equation}\label{Ai} A_i:=  \frac{2\sqrt \pi}{\Gamma\tonde{\frac{d+1} 2}}\E^d_{\GOI\tonde{\frac 1 2}} \left[  \Pi(\boldsymbol{\lambda}) \mathbb I_{\mathcal O_i}(\boldsymbol\lambda) \right],
	\end{equation}
	we obtain the claim for $\kappa'(1)\geq 1$.  Otherwise, if $\kappa'(1)<1$ then $1+\eta_L \to 1 + \eta_1 (1-\kappa'(1) )$. Using Lemma~\ref{goi_stima_1/2} again, we have
	$$\E^d_{\GOI\tonde{\frac{1+\eta_L}2}}
	\left[  \Pi(\boldsymbol{\lambda}) \mathbb I_{\mathcal O_i}(\boldsymbol\lambda) \right] \to \E^d_{\GOI\tonde{\frac{  1+\eta_1 (1-\kappa'(1))} 2}}
	\left[  \Pi(\boldsymbol{\lambda}) \mathbb I_{\mathcal O_i}(\boldsymbol\lambda) \right]. $$
	Setting
	\begin{equation}\label{Bi}
		\begin{aligned}B_i(\kappa'(1),\kappa''(1))
			:=\frac{2\kappa''(1)\sqrt{\pi} \, \E^d_{\GOI\tonde{(1+\eta_1(1-\kappa'(1))/2)}}
				\left[  \Pi(\boldsymbol{\lambda}) \mathbb I_{\mathcal O_i}(\boldsymbol\lambda) \right] }{(\kappa''(1) +\kappa'(1)  -\kappa'(1)^2 ) \Gamma \tonde{\frac{d+1} 2 }} ,
		\end{aligned}
	\end{equation}
	we have the claim also for $\kappa<1$.

	\begin{remark}\label{remt1}
		We note that all the $\mathrm{GOI}$ matrices involved are well defined for every $L$ and also in the limit. Indeed, by~\Cref{lemma_gen} we have
		$\kappa_L'(1) = \Var(E_i T_L(p))$ and  $ \kappa_L''(1) = \Var  T_L(p))$
		for every $i\neq j$. So that $\kappa_L'(1), \kappa_L''(1) \geq 0$ and hence $\eta_L \geq 0$. It follows that
		$$	1 + d \frac{1+\eta_L}{2} \geq1 + \frac{d}{2} > 0 . $$
		Since $\eta_L \geq 0$ for all $L$, also $\eta := \lim_{L\to\infty} \eta_L \geq 0$, and therefore
		$$1 + d \frac{1+\eta}{2} \geq 1 + \frac{d}{2} > 0$$
		Thus the non-degeneracy condition for $\mathrm{GOI}$ matrices is always satisfied.
	\end{remark}

	\subsection{Techinical lemmas}\label{technical}
	Before stating the next lemma, we point out that the argument of the proof
	would remain valid for a more general class of indicator functions. 		Nevertheless, for our purposes we only need the specific form with
	$\mathbb{I}_{\mathcal O_i}$, which is precisely what is required to justify the		passage of the limit under the integral sign in the proof of~\Cref{th2}; therefore we do not formulate the lemma in its
	full generality.

	\begin{lemma}\label{goi_stima_1/2}Let $(a_n \; | \;  n \in \mathbb{N})$ be a real sequence such that $a_n \to a \in \mathbb{R}$ whit $1 + da_n, 1 + da>0$. Then
		$$\lim_{n\to \infty}  \E_{\GOI\tonde{a_n}}^d  \quadre{\Pi (\boldsymbol{\lambda}) \mathbb I_{\mathcal O_i} (\boldsymbol \lambda)}=  \E_{\GOI(a)}^d  \quadre{\Pi (\boldsymbol{\lambda}) \mathbb I_{\mathcal O_i} (\boldsymbol \lambda)}. $$
	\end{lemma}

	\begin{proof} Since the  density  of the ordered eigenvalues of a  $\GOI(a)$ matrix is a function continuous in $a$, for all $\boldsymbol{\lambda}\in \R^d$ we have
		$$ \lim_{n\to  \infty}  f_{a_n}(\boldsymbol{\lambda})\Pi(\boldsymbol{\lambda} ) \mathbb I_{\mathcal O_i}(\boldsymbol{\lambda}) =  f_{a}(\boldsymbol{\lambda}) \Pi(\boldsymbol{\lambda})\mathbb I_{\mathcal O_i}(\boldsymbol{\lambda}) . $$
		We now show that this pointwise convergence is dominated by an integrable function. 	We observe that
		$$		 \lim_{n\to \infty} \frac{ a_n }{2(1+da_n)}= \frac{1+a}{2(1+da)} .$$
		Therefore, 		for every $\varepsilon>0$,  there exists an integer  $n_\varepsilon$  such that for all  $n \ge
		n_\varepsilon$, the following inequalities hold:
		\begin{align*}\frac{1}{\sqrt{1+da_n}} \leq &\sqrt{\frac{1}{1+ da}}+ \varepsilon ,\\
			\frac{ a_n }{2(1+da_n)} \leq& \frac{a}{2(1+da)} + \varepsilon^2 .
		\end{align*}
		In particular, choosing $\varepsilon:= \big(4d(1+da)\big)^{-1/2}$, we obtain that  definitely  in~$n$,
		\begin{align*}
			\frac{1}{\sqrt{1+da_n}} \leq &\frac{2\sqrt d +1}{\sqrt{4d(1+da)}}  ,\\
			\frac{ a_n}{2(1+da_n)} \leq&  \frac{2da+1}{4d(1+da)} .
		\end{align*}
		In the following, $C$  denotes a positive constant depending only on
		$d$ and  $a$; its value may change from line to line.
		\begin{align*}
			f_{a_n}(\boldsymbol{\lambda})\Pi(\boldsymbol{\lambda} ) \mathbb I_{\mathcal O_i}(\boldsymbol{\lambda}) &\leq C\exp\tonde{ -\frac { \| \boldsymbol \lambda \|^2} 2 + \frac{ 2da+1}{4d(1+da)} (\bold 1^\top \boldsymbol \lambda)^2}\Pi(\boldsymbol \lambda ) \Delta(\boldsymbol \lambda ) \mathbb I_{\mathcal O_i\cap\mathcal O}(\boldsymbol{\lambda})\\
			& \leq C\exp\tonde{ -\frac 1 2 \tonde{ 1 - \frac{2da+1}{2(1+da)}} \| \boldsymbol \lambda \|^2 }\Pi(\boldsymbol{\lambda}) \Delta(\boldsymbol \lambda ) \mathbb I_{\mathcal O_i\cap\mathcal O}(\boldsymbol{\lambda})\\
			& = C \exp\tonde{ -\frac{ \| \boldsymbol \lambda \|^2 }{4(1+da)}}\Pi(\boldsymbol{\lambda}) \Delta(\boldsymbol \lambda ) \mathbb I_{\mathcal O_i\cap\mathcal O}(\boldsymbol{\lambda})
		\end{align*}
		where the second inequality follows from the Cauchy-Schwarz inequality. Let $g(\boldsymbol{\lambda})
		$ denote the right-hand side of the last equality, and let $\boldsymbol{\widetilde Z} = (\widetilde Z_1, \dots, \widetilde{Z_d})$ be a random vector with $\boldsymbol{\widetilde Z}  \sim \mathcal N(0,2(1+da))I_d)$. So
		\begin{equation}\label{bound1}
			\begin{aligned} \int_{\R^d} g(\boldsymbol \lambda) \di \boldsymbol \lambda &= C\E\quadre{\Pi\big( \boldsymbol {\widetilde Z}\big ) \Delta\big( \boldsymbol{\widetilde Z} \big)  \mathbb I_{\mathcal O_i \cap \mathcal O}\big( \boldsymbol{\widetilde Z}\big) } \\
				& \leq C E \quadre{\Pi(\boldsymbol{Z}) \Delta(\boldsymbol{Z}) \mathbb I_{\mathcal O_i \cap \mathcal O}(\boldsymbol{Z}) } \leq C \E \quadre{\Pi(\boldsymbol{Z}) \Delta(\boldsymbol{Z}) \mathbb I_{\mathcal O}(\boldsymbol{Z}) }\\
				& \leq C \E\quadre{\prod_{j=1}^d |Z_j|\prod_{h=j+1}^d |Z_j - Z_h| \mathbb I_{Z_1 \leq \cdots Z_{i} < 0 < Z_{i+1} \leq \cdots \leq Z_d}}  \\
			\end{aligned}
		\end{equation}
		where $\boldsymbol{ Z}\sim \mathcal N(0,I_d)$ and we have used the identities
		\begin{align*} \Pi(a\boldsymbol x) = &a^d \Pi(\boldsymbol x),\\ \Delta(a\boldsymbol x)  =& a^{(d^2-d)/2}  \Delta(\boldsymbol x).
		\end{align*}
		To bound the last quantity appearing in~\eqref{bound1},
		one can observe that
		\begin{align*}
			|Z_j| \mathbb I_{\mathcal O}(\boldsymbol{Z}) &\leq \max \{|Z_1|, |Z_d| \}.
		\end{align*}
		Hence
		$$ \Pi(\boldsymbol{Z})\mathbb I_{ \mathcal O}(\boldsymbol{Z})  \leq \max \{|Z_1|, |Z_d| \}^d $$
		and similarly
		\begin{equation}\label{in_GOI0}
			\Delta(\boldsymbol{Z})\mathbb I_{ \mathcal O}(\boldsymbol{Z})\leq  C \max \{|Z_1|, |Z_d| \}^{(d^2-d)/2}.
		\end{equation}
		Therefore, using~\eqref{bound1}, we obtain
		\begin{align*}
			\int_{\R^d} g(\boldsymbol \lambda) \di \boldsymbol \lambda \leq& C \E\quadre{\max\s{|Z_1|, |Z_d|}^{(d^2+d)/2}
			}\\
			\leq &C
			\E\quadre{ |Z_1|^{(d^2+d)/2}} <\infty
		\end{align*}
		where the last inequality follows since $Z_1 \sim Z_d$ and since a Gaussian random variable has all absolute moments finite.
		One can conclude using the dominated convergence theorem.

	\end{proof}
	The next lemma provides an explicit change of variables that transforms the expectation under the $\mathrm{GOI}$ density into an expectation under multivariate Gaussian density.  As a consequence, one can estimate the $\mathrm{GOI}$ expectation by generating samples from a multivariate Gaussian distribution and applying a Monte Carlo method.  Instead of sampling from the $\mathrm{GOI}$ distribution directly, we can sample from a standard Gaussian vector with a known covariance structure  and compute the corresponding average of the transformed integrand (see~\Cref{numerical_evidence} for more details)
	\begin{lemma}\label{GOIc}Let $1+dc>0$ then
		$$ \E_{\GOI(c)}^d [ g(\boldsymbol\lambda)]=
		\frac{(2\pi)^{d/2}}{K_d } \E \quadre{ g(\bold Z)  \Delta(\bold Z)\mathbb I_{\mathcal O}(\bold Z)} $$
		where $\bold Z \sim \mathcal N(0,\Theta_{c;d})$ with $\Theta_{c;d}:= I_d + c \bold 1 \bold 1^\top $.
	\end{lemma}
	\begin{proof}From~\eqref{dens} we obtain
		\begin{equation}\label{1} f_c(\boldsymbol{\lambda})=   	\frac{1}{K_d \sqrt{1 + dc}} \exp\tonde{-\frac 1 2 \boldsymbol \lambda^\top A_{c;d} \boldsymbol{\lambda}} \Delta(\boldsymbol\lambda) \mathbb I_{\mathcal O}(\boldsymbol{\lambda})  .
		\end{equation}
		where
		$$ A_{c;d}:= I_d - \frac{c}{1+dc} \bold 1 \bold 1 ^\top .$$
		We note that for every $\boldsymbol  x \in \R^d$
		\begin{align*}A_{c;d} \Theta_{c;d}\boldsymbol  x &= \tonde{I_d - \frac{c}{1+dc} \bold 1 \bold 1^T} \tonde{I_d + c \bold 1 \bold 1 ^T } \boldsymbol  x\\
			& =  \boldsymbol{ x} +\tonde{c -\frac{c}{1+dc}- \frac{c^2 (\boldsymbol{1}^\top \boldsymbol{1}) }{1+dc}}(\boldsymbol{1}^\top \boldsymbol{ x})  \bold 1 = \boldsymbol x
		\end{align*}
		where the last equality follows from the identity $\bold 1^\top \bold 1 = d $.  Since this holds for  any arbitrary $\boldsymbol x$, we conclude that  $\Theta_{c;d}= A_{c;d}^{-1}$ , indeed $A_{c;d}$	is symmetric. \\
		Let us compute the eigenpairs of $\Theta_{c;d}$. 	Let  $\boldsymbol {v_1}, \dotsc, \boldsymbol{v_{d-1}}$ be a basis of $Span(\bold 1)^\perp$ then
		$$ \Theta_{c;d} \boldsymbol{v_i}= \boldsymbol{v_{i}}, \quad i = 1, \dots, d-1$$
		and
		$$ \Theta_{c;d}\bold 1   = \bold 1   + c  (\bold 1^\top \bold 1  )\bold 1  =  (1+dc) \bold 1  . $$
		We have thus shown that
		$1$ is an eigenvalue of $\Theta_{c;d}$ with geometric multiplicity  $d-1$, and that
		$1+dc $ is a simple eigenvalue. In particular, $\Theta_{c;d}$ has positive eigenvalues and thus it is positive definite.
		Recalling that the density of a centered Gaussian vector with covariance matrix $A\in \R^{d\times d}$	is given by
		$$ \psi_A(\boldsymbol{x}) = \frac{1}{(2\pi)^{d/2} \sqrt{\det(A)}} \exp\tonde{ - \frac 1 2 \boldsymbol{x}^\top A^{-1} \boldsymbol{x}}, $$
		we can rewrite~\eqref{1} as
		$$  f_{c}(\boldsymbol \lambda)  =\Delta(\boldsymbol\lambda) \mathbb I_{\mathcal O}(\boldsymbol{\lambda})
		\frac{(2\pi)^{d/2}}{ K_d } \psi_{\Theta_{c;d}}(\boldsymbol \lambda)   $$
		and hence the claim.
	\end{proof}
	\subsection{Proof of~\Cref{th3}}\label{proof_th2}
	From the computation in~\Cref{verifica}, one can apply~\Cref{cheng2} and hence
	$$\E[\mathcal C_i(T_L,u)] = \frac{2\sqrt \pi}{
		\Gamma\tonde{\frac{d+1} 2 }\eta_L^{d/2}} \int_u^{+\infty} \phi(x) \E_{GOI\tonde{\frac{1+\eta_L - \xi_L}2}}\quadre{\Pi\tonde{\boldsymbol{\lambda}_{\xi_L,x}}
		\mathbb I_{\mathcal O_i} \tonde{\boldsymbol{\lambda}_{\xi_L,x}}} \di x$$
	where $\eta_L$ is as above and $\xi_L:= \kappa_L'(1)^2(\kappa_L''(1))^{-1}$.
	So, using~\eqref{as_kern1} and~\eqref{as_kern2} we have
	\begin{align*}
		&\xi_L = L^{-1} \eta_1 & \text{ if } \kappa'(1) = 1,\\
		&\xi_L = \eta_1 (\kappa'(1)-1) +o(1) & \text{ if } \kappa'(1) >1,\\
		&\xi_L = \eta_1(1-\kappa'(1))+ O(\kappa'(1)^L)& \text{ if } \kappa'(1) <1.
	\end{align*}
	, by a trivial computation,
	\begin{align*}
		\frac{1+ \eta_L - \xi_L} 2 & = \frac{1}{2} (1 + \eta_1 (1 - \kappa'(1)) := c(\kappa'(1),\kappa''(1))
	\end{align*}
	where $c(\kappa'(1), \kappa''(1))$  denote a constant that does not depend on $L$. For notational simplicity, we will omit its explicit dependence on $\kappa$ and write just $c$.
Let $\xi:= \lim_L \xi_L$, (we note that for $\kappa'(1)<1$ we have $\xi = 0 $ and for $\kappa'(1)>1$, $\xi = \kappa'(1)^2/\kappa''(1)$)  we prove that
\begin{align*} \lim_{L \to + \infty} \int_u^{\infty} \phi(x) \E^d_{\GOI(c)} \quadre{\Pi\tonde{\boldsymbol{\lambda}_{\xi_L,x}}\mathbb I_{\mathcal O_i}\tonde{\boldsymbol{\lambda}_{\xi_L,x} }} \di x  \\=\int_u^{\infty} \phi(x) \E^d_{\GOI(c)} \quadre{\Pi\tonde{\boldsymbol{\lambda}_{\xi,x}}\mathbb I_{\mathcal O_i}\tonde{\boldsymbol{\lambda}_{\xi,x} }} \di x .
\end{align*}
As previous, let $C$ and $C_1$ denote  positive constants depending only on
$d$ and $c$; their value may change from line to line and
%
let $\boldsymbol Z \sim \mathcal N(0, \Theta_{c;d})$  where $\Theta_{c;d}$ is as in~\Cref{GOIc}. Hence
\begin{align*}
	&\int_u^\infty \phi(x)\E^d_{\GOI(c)} \quadre{\Pi\tonde{\boldsymbol{\lambda}_{\xi_L,x}}\mathbb I_{\mathcal O_i}\tonde{\boldsymbol{\lambda}_{\xi_L,x} }} \di x  \\
	&= C\int_{u}^\infty \phi(x)
	\E_{\boldsymbol Z} \quadre{\Pi\tonde{\boldsymbol{Z}_{\xi_L,x}}\mathbb I_{\mathcal O_i}\tonde{\boldsymbol{Z}_{\xi_L,x} } \Delta(\boldsymbol{Z}) \mathbb I_{\mathcal O}(\boldsymbol{Z}) }\di x \\
	&=C \int_u^\infty  \int_{\R^d}  \phi(x)
	\Pi\tonde{\boldsymbol{\lambda}_{\xi_L,x}}\mathbb I_{\mathcal O_i}\tonde{\boldsymbol{\lambda }_{\xi_L,x} } \Delta(\boldsymbol{\lambda}) \mathbb I_{\mathcal O}(\boldsymbol{\lambda}) \psi_{\Theta_{c;d}}(\boldsymbol{\lambda}) \di \boldsymbol{\lambda} \di x.
\end{align*}
Let
$$g_L(\boldsymbol{\lambda},x) = \phi(x)
\Pi\tonde{\boldsymbol{\lambda}_{\xi_L,x}}\mathbb I_{\mathcal O_i}\tonde{\boldsymbol{\lambda }_{\xi_L,x} } \Delta(\boldsymbol{\lambda}) \mathbb I_{\mathcal O}(\boldsymbol{\lambda}) \psi_{\Theta_{c;d}}(\boldsymbol{\lambda})
$$
We note that if $\mathbb I_{\mathcal O}(\boldsymbol{\lambda}) \neq 0$  then
$$\tonde{\boldsymbol\lambda_{\xi_L,x}}_i \leq \max\left\{ \left| \lambda_1 - \frac{\xi_L x }{\sqrt 2}\right|, \left| \lambda_d- \frac{\xi_L x }{\sqrt 2}\right| \right\}, \qquad i = 1, \dots, d$$
and
\begin{equation}\label{eq3}
	\begin{aligned}\Pi\tonde{\boldsymbol\lambda_{\xi_L,x}} &\leq \max\left\{\left| \lambda_1 - \frac{\xi_L x }{\sqrt 2}\right|, \left| \lambda_d- \frac{\xi_L x }{\sqrt 2}\right| \right\}^d \\
		&\leq  \left| \lambda_1 - \frac{\xi_L x }{\sqrt 2}\right| ^d +  \left| \lambda_d- \frac{\xi_L x }{\sqrt 2}\right|^d \\
		&\leq C\tonde{ |\lambda_1|^d + |\lambda_d|^d + 2^{1-d/2} |\xi_L x|^d}
	\end{aligned}
\end{equation}
where in the last inequality we used that, for all $a,b\in \R$,  it holds that  $$|a-b|^d \leq C (|a|^d + |b|^d).$$
Now, since $\xi_L x \to \xi x $, we have $\xi_L x \leq \xi x  + 2^{-1}\varepsilon^{1/d}$ for $L$ large enough. Hence $$|\xi_L x |^d \leq C \tonde{ |\xi x |^d + 2^{-1}\varepsilon}$$
and thus, from equation~\eqref{eq3}, we obtain
$$\Pi\tonde{\boldsymbol\lambda_{\xi_L,x}} \leq C\tonde{ |\lambda_1|^d + |\lambda_d|^d + 2^{(d+2)/2} |\xi x |^d  + \varepsilon}. $$
Combining the previous inequality with~\eqref{in_GOI0}, it follows that
\begin{subequations}
	\begin{align}\nonumber &g_L(\boldsymbol \lambda, x)\\
		\nonumber
		&\leq C \phi(x) \tonde{ |\lambda_1|^d + |\lambda_d|^d + 2^{(d+2)/2} |\xi x |^d  + \varepsilon}\tonde{|\lambda_1|^{(d^2-d)/2}  +  |\lambda_d|^{(d^2-d)/2}} \psi_{\Theta_{c;d}}(\boldsymbol{\lambda})\\
		&=  C\phi(x) \tonde{ |\lambda_1|^d + |\lambda_d|^d  + \varepsilon}\tonde{|\lambda_1|^{(d^2-d)/2}  +  |\lambda_d|^{(d^2-d)/2}} \psi_{\Theta_{c;d}}(\boldsymbol{\lambda}) \label{s1}\\
		&\quad + C_1 |\xi x |^d   \tonde{|\lambda_1|^{(d^2-d)/2}  +  |\lambda_d|^{(d^2-d)/2}} \psi_{\Theta_{c;d}}(\boldsymbol{\lambda}) \label{s2}.
	\end{align}
\end{subequations}
So, if we denote by $\phi(x) h_1(\boldsymbol{\lambda})$ the expression in~\eqref{s1} and by
$h_2(\boldsymbol{\lambda})\phi(x)|x|^d $ the expression in~\eqref{s2}, we obtain  $$g_L(\boldsymbol{\lambda},x) \leq h_1(\boldsymbol{\lambda}) \phi(x) + h_2(\boldsymbol{\lambda}) \phi(x) |x|^d$$
and recalling the definition of $\Theta_{c;d}$ we have
\begin{align*}&\int_{u}^\infty \int_{\R^d}
	\big(h_1(\boldsymbol{\lambda}) \phi(x) + h_2(\boldsymbol{\lambda}) \phi(x) |x|^d \big) \di \boldsymbol{\lambda} \di x \\
	&= 2(1-\Phi(u))  \E\quadre{|Z_1|^{(d^2+d)/2} + |Z_1|^d |Z_2|^{(d^2-d)/2} + \varepsilon|Z_1|^{(d^2-d)/2}}\\
	&\quad +2 \E\quadre{|Z|^d \mathbb I_{[u,+\infty)}(Z)} \E[|Z_1|^{(d^2-d)/2}] <\infty
\end{align*}
for $Z\sim \mathcal N(0,1)$ and both $Z_1$ and $Z_2$ are zero-mean Gaussian random variables with variance equal to $1+c$ and the covariance between $Z_1$ and $Z_2$ equals $c$.  For $\kappa'(1)\leq 1$, since $\eta=0$, then $\boldsymbol{\lambda}_{\xi,x} = \boldsymbol{\lambda}$ and hence, by dominated convergence
\begin{align*}\lim_{L \to + \infty }\eta_L^{d/2} \E[\mathcal C_i(T_L,u)] &= \frac{2\sqrt \pi}{
		\Gamma\tonde{\frac{d+1} 2 }} \int_u^{+\infty} \phi(x) \E_{GOI\tonde{\frac{1+\eta_1(1-\kappa'(1))}2}}\quadre{\Pi\tonde{\boldsymbol{\lambda}}
		\mathbb I_{\mathcal O_i} \tonde{\boldsymbol{\lambda}}} \di x \\
	& =  \frac{2\sqrt \pi}{
		\Gamma\tonde{\frac{d+1} 2 }} \E^d_{GOI\tonde{\frac{1+\eta_1(1-\kappa'(1))}2}}\quadre{\Pi\tonde{\boldsymbol{\lambda}}
		\mathbb I_{\mathcal O_i} \tonde{\boldsymbol{\lambda}}} (1-\Phi(u))
	.
\end{align*}
If $\kappa'(1) = 1$ then  $\eta_L^{d/2} = L^{-d/2} \eta_1^{d/2}$ and hence
$$  \E[ C_i (T_L(u)] = L^{d/2} \Bigg(  \frac{2\sqrt \pi}{\eta^{d/2}
	\Gamma\tonde{\frac{d+1} 2 }} \E^d_{GOI(1/2)}\quadre{\Pi\tonde{\boldsymbol{\lambda}}
	\mathbb I_{\mathcal O_i} \tonde{\boldsymbol{\lambda}}} (1-\Phi(u))  + o(1)\Bigg)  . $$
Recalling the definition of $A_i$ (cf.,~\eqref{Ai}) we have that the claim holds in the sparse cases.  In the same way, one can prove the claim as for the sparse case.  Moreover, if $\kappa'(1)>1$, putting
\begin{equation}\label{Di}
	\begin{aligned}D_i(\kappa'(1), \kappa''(1), u):= &\frac{2 \sqrt{\pi}}{\eta^{d/2}(\kappa'(1)-1) \Gamma\tonde{ \frac{d+1} 2 }}  \cdot \\	&\int_u^\infty \E_{GOI\tonde{\frac{1+c(\kappa'(1)(1-\kappa'(1))}2}}\quadre{\Pi\tonde{\boldsymbol{\lambda}_{\eta,x}}
			\mathbb I_{\mathcal O_i} \tonde{\boldsymbol{\lambda}_{\eta,x}}}
	\end{aligned}
\end{equation}
where we recall that $\eta = \kappa'(1)^2/ \kappa''(1)$.

\begin{remark}
	As in~\Cref{remt1}, the $\mathrm{GOI}$ matrices involved are well defined for every $L$ and also in the limit. Indeed,
	\[
	1 + d \, \frac{1 + \eta_L - \xi_L}{2} > 0
	\quad \Longleftrightarrow \quad
	\eta_L - \xi_L > -\tfrac{1+d}{2}.
	\]
	We note that $\eta_L - \xi_L = -\gamma_L$, where $\gamma_L$ is given in~\eqref{deg_cond}. Since $\gamma_L < \tfrac{d+2}{2}$, the inequality above holds. Moreover, in the proof of~\Cref{th3} we have shown that
	\[
	\frac{1 + \eta_L - \xi_L}{2} = c\!\left(\kappa'(1), \kappa''(1)\right),
	\]
	which is independent of $L$. Hence, also the limiting $\mathrm{GOI}$ matrix is well defined.
\end{remark}

\section{Numerical Evidence}\label{numerical_evidence}
The numerical experiments reported in this section are designed primarily
to confirm the theoretical results established in the previous sections.
They also provide exploratory evidence in regimes that are not yet
covered by our main theorems, offering intuition for possible extensions
of the theory. All experiments can be reproduced from the repository \url{https://github.com/simmaco99/SpectralComplexity}. Simulations are conducted on $\S^2$ using the HEALPix package~\cite{gorski2005healpix}, which is widely recognized for its capabilities in handling pixelization of the sphere and performing analyses in spherical harmonics.
\smallskip

To validate the results of~\Cref{th2}, we follow the approach of~\cite{nostro2} and use the Gaussian activation function defined by $\sigma_a(u) = e^{-a^2 x^2/2}$, considering three different values for the parameter $a$. This choice allows us to explore all three regimes: for the low-disorder case, we set $a^2= 1$; for the sparse case, we take $a^2 = 1 + \sqrt{2}$; and for the high-disorder case, we choose $a^2= 9$.  To compute numerically the expected number of minima and maxima, for each value of $a$, we generate 1000 random neural networks with input on $\mathbb{S}^2$ and hidden layers of width $n = 1000$.  Figure~\ref{fig::critical} shows the Monte Carlo estimation of the number of minima and maxima. To obtain the theoretical values, one must know the value of $A_0$ given by~\eqref{Ai}; we use~\Cref{GOIc} to express it as an expectation over Gaussian variables, which we estimate using a Monte Carlo method. In Figure~\ref{fig::monte}, we illustrate how the estimated value of $A_0$ changes as the number of samples increases.
In~\Cref{fig::critical}(c), we observe that the theoretical and estimated values closely match as long as the depth remains below $L = 40$. This is due to the fact that, working in finite-precision arithmetic, all our random fields are truncated at frequency $\ell_{\max} = 1356$, and therefore we inevitably lose information for sufficiently large $L$. This loss is confirmed by~\Cref{fig::variance}, where we compute the percentage of variance explained using only the first 1356 frequencies. We note that this percentage starts to decrease precisely from $L = 40$ onward.

\smallskip

Since the results of~\Cref{th2} cannot be applied to irregular fields (such as fields generated by ReLU activation),~\Cref{fig::relu} and~\Cref{tab::ReLu} show how for a ReLU network the number of critical points increases with the resolution of the map (suggesting a potential divergence, as noted in~\Cref{rem}), while it remains approximately constant in the case of fields associated with more regular activation (as the Gaussian $\sigma_1$ or $\tanh$).

\begin{figure}[!htb]
	\centering
	\subfloat[Low-disorder case: $a^2=1$.]
	{	\resizebox{0.45\textwidth}{!}{\includegraphics{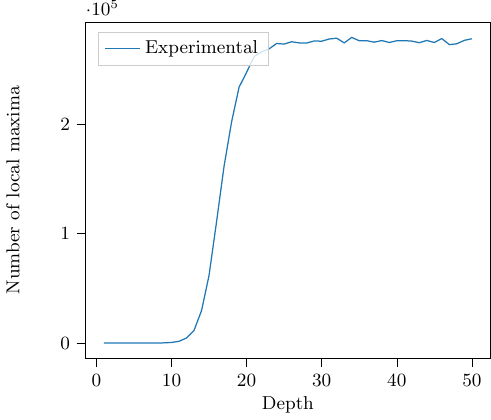}}}
	\subfloat[Sparse case: $a^2=1 + \sqrt 2$.]{
		\resizebox{0.45\textwidth}{!}{	\includegraphics{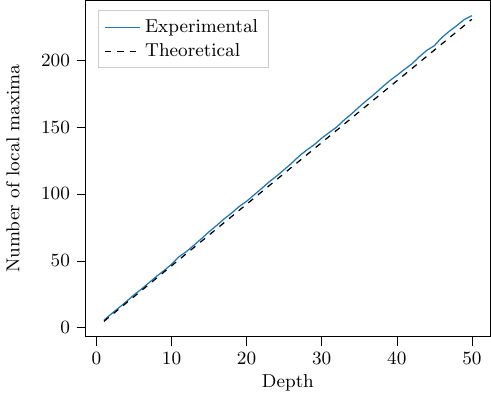}}}\\
	\subfloat[High-disorder case: $a^2=9$.]
	{
		\resizebox{0.45\textwidth}{!}{	\includegraphics{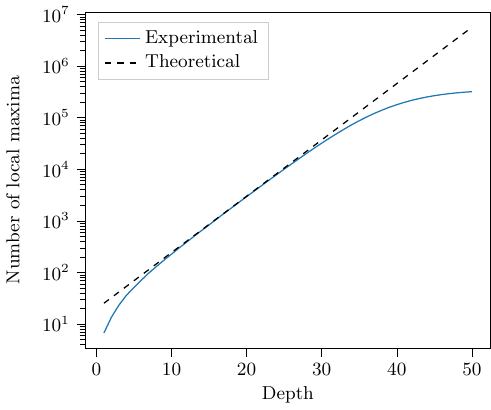}}
	}
	\caption{Number of minima and maxima points for the Gaussian activation functions with different values of $a$. The dashed lines represent our theoretical findings. These numbers are computed using  $1000$ Monte Carlo replicas.}
	\label{fig::critical}
\end{figure}

\begin{figure}[!htb]
	\centering
	\resizebox{0.45\textwidth}{!}{\includegraphics{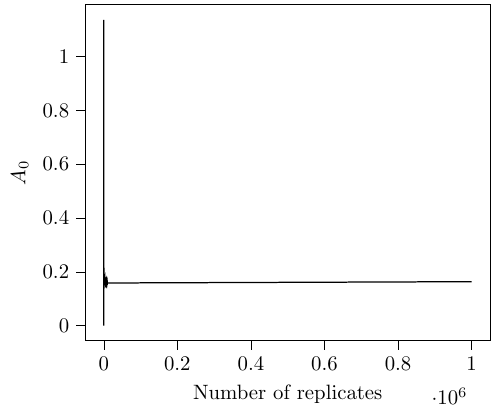}}
	\caption{Approximation of the $A_{0}$ as the number of Monte Carlo replicates increases for dimension $d=2$. }
	\label{fig::monte}
\end{figure}
\begin{figure}[!htbp]
	\centering
	\resizebox{0.45\textwidth}{!}{	\includegraphics{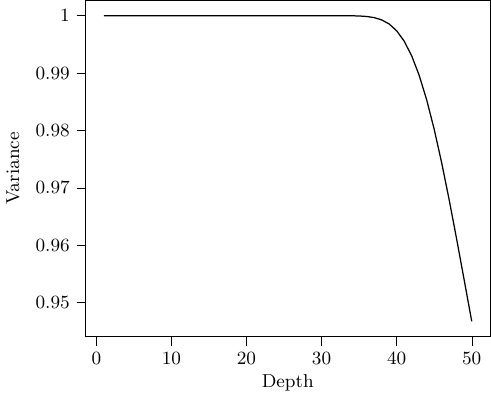}}
	\caption{Percentage of the variance explained by using the first $1536$ frequencies for  random neural networks on $\S^2$ with $\sigma_3$ as activation function for different values of depth $L$. To obtain the plot, we compute the angular power spectrum of this network using a Gauss--Legendre quadrature with $5000$ points.}
	\label{fig::variance}
\end{figure}

\begin{figure}[!htb]
	\centering
	\resizebox{0.45\textwidth}{!}{\includegraphics{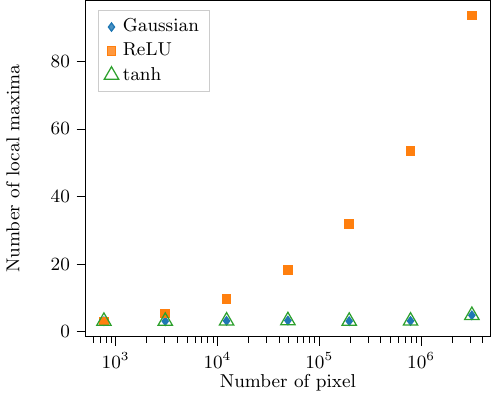}}
	\caption{Mean number of local maxima over $1000$ Monte Carlo replicates. The fields are shallow (one hidden layer) random neural networks with $1000 $ neurons. We use three activation functions: Gaussian with  $a=1+\sqrt 2$ (diamond), ReLU (square), and $\tanh$ (triangle).
	}
	\label{fig::relu}
\end{figure}

\begin{table}[!htbp]
	\centering
	\caption{Number of local minima and maxima for different activations and number of pixels. A resolution level $r$ corresponds to a HEALPix map with $12 \cdot 2^{2r}$ pixels.}
	\label{tab::ReLu}

	\begin{subtable}{\textwidth}
		\centering
		\caption{Local minima}
		\begin{tabular}{C C C C}
			\toprule
			Resolution & Gaussian & ReLU & $\tanh$ \\
			\midrule
			3 & 3.108 & 3.085 & 1.548 \\
			4 & 3.136 & 5.367 & 1.569 \\
			5 & 3.260 & 9.900 & 1.563 \\
			6 & 3.182 & 18.10 & 1.620 \\
			7 & 3.146 & 31.28 & 1.661 \\
			8 & 3.176 & 53.51 & 1.677 \\
			9 & 4.750 & 93.40 & 1.570 \\
			\bottomrule
		\end{tabular}
	\end{subtable}
	\vspace{1em} 

	\begin{subtable}{\textwidth}
		\centering
		\caption{Local maxima}
		\begin{tabular}{C C C C}
			\toprule
			Resolution & Gaussian & ReLU & $\tanh$ \\
			\midrule
			3 & 3.082 & 3.027 & 1.552 \\
			4 & 3.094 & 5.313 & 1.570 \\
			5 & 3.174 & 9.710 & 1.563 \\
			6 & 3.270 & 18.17 & 1.619 \\
			7 & 3.102 & 31.80 & 1.661 \\
			8 & 3.150 & 53.49 & 1.677 \\
			9 & 4.820 & 93.48 & 1.575 \\
			\bottomrule
		\end{tabular}
	\end{subtable}\\

\end{table}

\subsection*{Acknowledgements}
The author wishes to thank Professor Domenico Marinucci  for his helpful suggestions and insightful discussions.

This work was partially supported by the MUR Excellence Department Project MatMod@TOV awarded to the Department of Mathematics, University of Rome Tor Vergata, CUP E83C18000100006. We also acknowledge financial support from the MUR 2022 PRIN
project GRAFIA, project code 202284Z9E4.

The author is a member of the Gruppo Nazionale per l’Analisi Matematica, la Probabilit\`a e le loro Applicazioni (GNAMPA), which is part of the Istituto Nazionale
di Alta Matematica (INdAM).
\newpage
\bibliographystyle{splncs04}
\bibliography{ref}
\appendix

\numberwithin{equation}{section}

\section{Proof of~\Cref{prop::1}}\label{appA}
To prove the proposition, we adapt the proof of~\cite[Proposition 3.14]{nostro2}.
\subsection{Outline of the proof}

Let $\varphi_N:\, \S^d \setminus \s{N} \to \R^d$ denote the stereographic projection from the north pole $N$  and $\varphi_S:\, \S^d \setminus \s{S} \to \R^d$ that from the south pole $S$. Then,
\begin{align*}
	&\varphi_N\tonde{\begin{pmatrix} x_1 , \dots , x_{d+1} \end{pmatrix}} = \begin{pmatrix}
		\frac{x_1}{1-x_{d+1}} , 	\frac{x_2}{1-x_{d+1}} , \dots , \frac{x_d}{1-x_{d+1}}
	\end{pmatrix} , \\
	&	\varphi_S\tonde{\begin{pmatrix} x_1 , \dots , x_{d+1} \end{pmatrix}} = \begin{pmatrix}
		-\frac{x_1}{1+x_{d+1}} , 	\frac{x_2}{1+x_{d+1}}, \dots ,  \frac{x_{d}}{1+x_{d+1}}
	\end{pmatrix} .
\end{align*}
Since $ \{ (\S^d \setminus \s{N}, \varphi_N), (\S^d \setminus \s{S}, \varphi_S) \} $ is an atlas of $\S^d$, the fact that $ T \in C^1(\S^d) $ almost surely is equivalent to the fact that $f_N = T \circ \varphi_N^{-1}$ and $f_S = T \circ \varphi_S^{-1}$ are in $C^2(\R^d)$ almost surely.  Note that, in~\cite{nostro2}, under the assumption $\mathrm{CRI}>1$, the authors proved that $T\in C^1$ almost surely. Following the notation in~\cite{nostro2},  for every $x\in \R^d$  and for every $i=1, \dots, d$, we define
\begin{equation}\label{derivata}
	D^i f_N(s):= \lim_{h \to 0} \frac{f_N(s+he_i) - f_N(s)}{h} ,
\end{equation}
where $e_i$ is the $i$-th vector of the canonical basis of $\R^d$ and  the limit holds almost surely. We also recall that
\begin{equation}\label{coviii} \E[ f_N(t)f_N(s)] = \kappa(x_{t,s}), \quad  x_{t,s}:=1- \frac{\nn{t-s}}{(1+\nn t) (1+\nn s)}.
\end{equation}

The main steps of the proof are as follows:
\begin{enumerate}[label = \roman*)]
	\item \label{it:i}First, we prove that $D^i f_N$ has a mean-square derivative in every direction; that is, for each $j=1, \dots, d$  there exists a random field $D^{ji} f_N$ such that, for every $s \in \R^d$,
	\begin{equation}\label{derivata2}
		\lim_{h \to 0} \frac{D^if_N(s+he_i) - D^if_N(s)}{h} = D^{ji} f_N(s),
	\end{equation}
	where the limit holds in $L^2$.
	\item \label{it:ii} Second, we prove that, for every $s \in \R^d$ and $\varepsilon > 0$,
	\begin{equation}\label{somma}
		\sum_{n=1}^\infty \P\tonde{\left| n \tonde{D^i f_N(s+e_j/n) - D^if_N(s)} - D^{ji} f_N(s) \right| > \varepsilon} < \infty,
	\end{equation}
	and therefore, by the Borel--Cantelli lemma, the limit in~\eqref{derivata} also holds  almost surely.

	\item \label{it:iii} Then we prove that there exists $\eta,\zeta_1, \zeta_2, K > 0$ such that,  $\forall t, s \in \R^d, \n{t-s} < \eta$,
	\begin{equation}\label{kol}
		\E\quadre{| D^{ji} f_N(s) - D^{ji} f_N(t)|^{\zeta_1}} \leq K \n{t-s}^{d+\zeta_2}.
	\end{equation}
	Thus, by Kolmogorov’s continuity theorem \cite[Theorem 3.23]{kallenberg2002foundations}, there exists a Hölder continuos  version of $D^{ji} f_N$.

	\item \label{it:iv}Finally, by reviewing the proofs of items \ref{it:i}-\ref{it:iii}, we observe that the arguments rely solely on the expression for the covariance function: $\Sigma_{f_N} $ of the field $f_N$,
	$$
	\Sigma_{f_N}(s,t) = \E[f_N(s)f_N(t)] \,.
	$$
	Since $\Sigma_{f_N} = \Sigma_{f_S}$, the results established for $f_N$ also hold for $f_S$.
\end{enumerate}

We will also make use of the following standard lemmas.

\begin{lemma}[footnote~20, p.~23~\cite{book:73784}]
	\label{tecnico1}Let $(X_n)_{n\in \N}$ be a sequence of random variables with finite second moments. Then
	$$ X_n \xrightarrow{L^2} X \quad \Longleftrightarrow \quad \lim_{n,m\to + \infty}  \E[X_n X_m] <\infty  .$$
\end{lemma}

\begin{lemma}\label{tecnico2} Let $(X_n)_{n\in \N}$ be a sequence of random variables with finite second moments and assume that $X_n \xrightarrow{L^2}X$. Then for any $Y$ with finite second moment,
	$$\E[ X_n Y]\xrightarrow{L^2}\E[XY] . $$
\end{lemma}

\subsection{Proof of~\ref{it:i}: seconds derivatives exist in the mean-square sense}
Let us abbreviate $ f_N:=g$. Fix $t\in \R^d$  and $i,j\in \s{1, \dots, d}$. Define
\begin{align*}A_{hkpq}:=  &\E\Bigg[ \Big( g(t+he_i+ke_j) -g(t+he_i) -g(t+ke_j) + g(t) \Big) \\
	& \qquad  \Big(g(t+pe_i+qe_j) -g(t+pe_i) -g(t+qe_j) + g(t) \Big) \Bigg].
\end{align*}
In this section we prove that
$$ A_{hkpq} = hkpq \int_{0}^1\int_0^1 \int_0^1 \int_0^q F(x+h,y+k,z+p,k+q) \di x \di y \di z\di k $$
where $\|F\|_\infty \leq M $ and $F(x,y,z,k) \to 0 $ as $h,k,p,q\to 0 $. By dominated convergence, this implies that
$$ \lim_{hkpq\to 0 } \frac{A_{hkpq}}{hkpq} < \infty $$
and hence, by~\Cref{tecnico1}, \eqref{derivata2} holds in $L^2$.

\medskip
Let
$$E(x,y,z,k):= 1 - \frac{\nn{ (x- z)e_i  + (y-k)e_j}} {(1+\nn{t+xe_i+ ye_j})(1+\nn{t+ze_i +ke_j})}.$$
Since $E(x,y,z,k) \leq 1 $ and $E(x,y,z,k) \to 1 $ as $x,y,z,k\to 0 $, there exists $\delta>0$ such that
$$ \frac{1}2 \leq  E(x,y,z,k)\leq 1, \qquad \forall (x,y,z,k) \in B(0, 5\delta)$$
where $B(u,r)$ denotes the open ball in $\R^4$ with radius equal to $r$ and centred in $u\in \R^4$.  Hence, the map is well defined:
$$ G:\, [-\delta, \delta]^4 \to [-1,1] \qquad G(x,y,z,k) = \kappa(E(x,y,z,k))\,. $$
Now, for $h,k,p,q\in [-\delta, \delta]$, using~\eqref{coviii} we have:
$$A_ {hkpq} =   \Delta G(p,q|h,k) - \Delta  G(p,0|h,k) - \Delta(0,q|h,k)  + \Delta(0,0|h,k) . $$
where
$$  \Delta G(p,q|x,y) = G(x,y,p,q)- G(x,0,p,q) - G(0,y,p,q) + G(0,0,p,q) .  $$
The idea of the proof is to show that the following maps $ G_4(x,y,\bullet,k)$, $G_{34}(x, \bullet , k,z)$, $G_{234}(\bullet , y,k,z)$
from $[-\delta, \delta] \to \R$  are regular enough  to apply the Fundamental Theorem of Calculus.  Here, $G_{ab}$ denotes the second derivative $ \partial_a \partial_b G $.

Since $\kappa\in C^\infty((-1,1))$ and $E\in C^\infty((-1,1))$ the maps are $C^1$ in every point such that $E(x,y,z,k)\neq 1$.
\begin{itemize}
	\item Since $\kappa\in C^2([-1,1])$ and $E\in C^\infty([-\delta, \delta]^4)$, then $ G_4(x,y,\bullet ,k) \in C^1([-\delta, \delta])$ for every $x,y,k$.
	\item Let us study $G_{34}(x,\bullet,z,k)$ for every fiexed $x,k,z\in [-\delta,\delta]$. Since $\kappa\in C^2([-1,1])$ the function is continuous; 	setting $\boldsymbol{v}  = (x,y,z,k)$, then for every $y$ such that $ E(x,y,z,k) \neq 1$ we have
	\begin{equation}\label{bas-1}
		\begin{aligned}G_{234}(\boldsymbol v ) = &	 \kappa'''(E(\boldsymbol v)) E_2(\boldsymbol v)E_3(\boldsymbol v)E_4(\boldsymbol v) +H(\boldsymbol{v})
		\end{aligned}
	\end{equation}
	where
	\begin{equation}\label{bas0}
		\begin{aligned}
			H(\boldsymbol{v}) = &\kappa''(E(\boldsymbol v)) \Bigg( E_{23}(\boldsymbol{v}) E_4(\boldsymbol{v} ) + E_3 (\boldsymbol{v} )E_{24}(\boldsymbol{v} ) + E_2(\boldsymbol{v} )  E_{34}(\boldsymbol{v} ) \Bigg) \\
			&+ \kappa'(E(\boldsymbol v)) E_{234} (\boldsymbol{v} )    .
		\end{aligned}
	\end{equation} 	Using the regularity of $\kappa$, there exists $M_1\in \R$ such that
	\begin{equation}\label{b0}
		\begin{aligned} M_1:= \max_{\boldsymbol{v}\in [-\delta,\delta]^4}\Big| H(\boldsymbol{v})\Big|.
		\end{aligned}
	\end{equation}
	We note that the unique $y$ such that  $E(x,y,z,k) =1$  is $ k+  z-x$ if $i=j$ or $ k$ if $z=k$ and $i\neq j$.  Let
	$$ \boldsymbol{\xi }:=  (x,z-x+k,z, k) \delta_{ij} + (x,k,x,k) (1-\delta_{ij}).$$
	For any $\rho\in [-\delta, \delta]$ such that $|k  + (z-x)\delta_{ij}+\rho|\leq \delta$, let $\boldsymbol{\xi}_\rho:=\boldsymbol{\xi} + \rho e_2$.  So $E(\boldsymbol{\xi}_\rho)<1$ for any $\rho\neq 0$.  A trivial computation shows that
	$$ E_1(\boldsymbol{\xi}_\rho)= E_2(\boldsymbol{\xi}_\rho)= \frac{ \rho f(\rho)  }{(1+\nn{t+(z+k+\rho)e_i })^2(1+\nn{t+(z+k)e_i })},$$
	$$ E_3(\boldsymbol{\xi}_\rho)= E_4(\boldsymbol{\xi}_\rho)= -\frac{ \rho f(\rho)}{(1+\nn{t+(z+k)e_i })^2(1+\nn{t+(z+k+ \rho )e_i })} $$
	where
	$$ f(\rho):= - 4\Big( 1+\nn{t}  +2 t_i(z+k) +(z+k)^2+\rho(z+k +t_i)\Big)  . $$
	In particular, for any $a\in \s{1,\dots, 4}$, $E_a(\boldsymbol{\xi}_\rho) = O(\rho)$ as $\rho \to 0$ and using the definition of $\mathrm{CRI}$ we have:
	$$ \kappa'''(E(\boldsymbol \xi_\rho)) = O(\rho^{2(\beta-3)}), $$
	indeed,  $E(\boldsymbol \xi_\rho)= O(\rho^2)$.
	Let us recall that $\beta>2$, thus
	$$ \lim_{\rho \to 0}  \kappa'''(E(\boldsymbol \xi_\rho)) E_2(\boldsymbol\xi_\rho)E_3(\boldsymbol \xi_\rho)E_4(\boldsymbol \xi_\rho) = 0,  $$
	and there exists $\varepsilon = \varepsilon(\xi)$ such that
	\begin{equation}\label{bas1}  |\kappa'''(E(\boldsymbol \xi_\rho)) E_2(\boldsymbol\xi_\rho)E_3(\boldsymbol \xi_\rho)E_4(\boldsymbol \xi_\rho)|   \leq 1 \qquad |\rho|< \varepsilon.
	\end{equation}
	Otherwise, since $\kappa'''$ is continuous in $(-1,1)$, there exists $M_2$ such that for any $\rho> \varepsilon$ (such that $\boldsymbol{\xi}_\rho$ is well-defined),
	\begin{equation}\label{bas2}|\kappa'''(E(\boldsymbol{\xi}_\rho) E_2(\boldsymbol\xi_\rho)E_3(\boldsymbol \xi_\rho)E_4(\boldsymbol \xi_\rho) |< M_2 .
	\end{equation}
	Combing~\eqref{bas-1} with~\eqref{bas0}, \eqref{bas1}, and \eqref{bas2}, we have
	$$ \max_{\boldsymbol{w}\in [-\delta, \delta]^4 }  		  |G_{234}(\boldsymbol{w})|   \leq M_1 + M_2 + 1:= M .$$
	Therefore, $G_{34}(x, \bullet, z, k)$ is Lipschitz and thus absolutely continuous.
	\item Let us $y,z,k$ be fixed point in $[-\delta,\delta]$. We
	study the map $G_{234}(\bullet, y,z,k)$.We first observe that if $y \neq k$ and $i \neq j$, then $E(y,z,k) < 1$,
	and hence the function is $C^\infty$.
	Therefore, in the case $i \neq j$, we shall implicitly assume $y = k$.  Moreover, the unique solution $x$ of the equation $E(x,y,z,k) = 1$ is given by
	\[
	x = z + (k-y)\,\delta_{ij}.
	\]

	Using~\eqref{bas-1}, it follows that the function  $ x \mapsto G_{234}(x,y,z,k)$ 	 is well defined for all $x \neq z+(k-y)\delta_{ij}$. Set
	\[
	\boldsymbol{\eta}
	:= (\,z-y+k,\, y,\, z,\, k\,)\,\delta_{ij}
	+ (\,z,\, k,\, z,\, k\,)\,(1-\delta_{ij}) \, .
	\]
	For any $\rho\in [-\delta, \delta]$ such that $ |z + (k-y) \delta_{ij}|\leq \delta$, let $\boldsymbol{\eta}_\rho:=\boldsymbol \eta + \rho e_1 $.  As in the previous step, one can check that,  as $\rho \to 0 $
	\begin{equation}\label{decaay} E_a (\boldsymbol{\eta}_\rho)  = O(\rho), \quad  \kappa'''(E(\boldsymbol{\eta}_\rho))  = O(\rho^{2(\beta-3)}) .
	\end{equation}
	Now $\beta>2$, Thus, we can extend $G_{234}(x,y,z,k)$ by continuity by setting
	\[
	G_{234}(\boldsymbol{\eta}) = H(\boldsymbol{\eta}).
	\]
	By the chain rule, we obtain
	\begin{equation}\label{derii}
		\begin{aligned}
			G_{1234}(\boldsymbol{v}) =& H_2 (\boldsymbol{v}) +  \kappa''''(E(\boldsymbol{v})) E_1(\boldsymbol{v})E_2(\boldsymbol{v})E_3(\boldsymbol{v})E_4(\boldsymbol{v}) \\
			&  + \kappa'''(E(\boldsymbol{v})) \Bigg(  E_3(\boldsymbol{v}) E_4(\boldsymbol{v})  E_{12}(\boldsymbol{v}) +E_2 (\boldsymbol{v}) E_{4}(\boldsymbol{v}) E_{13}(\boldsymbol{v})  \\
			& \quad + E_2 (\boldsymbol{v})E_3(\boldsymbol{v}) E_{14}(\boldsymbol{v})  + E_{1} (\boldsymbol{v}) E_4(\boldsymbol{v})  E_{23}(\boldsymbol{v}) \\
			& \quad + E_{1} (\boldsymbol{v}) E_3(\boldsymbol{v})E_{24}(\boldsymbol{v}) + E_{1} (\boldsymbol{v}) E_2 (\boldsymbol{v})E_{34}(\boldsymbol{v}) \Bigg) \\
		\end{aligned}
	\end{equation}
	where
	\begin{align*}
		H_2(\boldsymbol{v}):= &    \kappa'(E(\boldsymbol{v})) E_{1234}(\boldsymbol{v})  + \kappa''(E(\boldsymbol{v}))\Bigg( E_{123}(\boldsymbol{v}) E_4(\boldsymbol{v}) + E_{23}(\boldsymbol{v}) E_{14}(\boldsymbol{v}) \\
		&\quad + E_{13}(\boldsymbol{v}) E_{24}(\boldsymbol{v}) + E_3(\boldsymbol{v}) E_{124}(\boldsymbol{v}) + E_{12}(\boldsymbol{v}) E_{34} (\boldsymbol{v})\\
		& \quad + E_2(\boldsymbol{v}) E_{134}(\boldsymbol{v}) + E_1(\boldsymbol{v})E_{234}(\boldsymbol{v})\Bigg) .
	\end{align*}
	Using the regularity of $\kappa$, exists $N_1\in \R$ such that
	\begin{equation*}
		\begin{aligned} N_1:= \max_{\boldsymbol{v}\in [-\delta,\delta]^4}\Big| H_2(\boldsymbol{v})\Big|.
		\end{aligned}
	\end{equation*}
	As $\rho \to 0$,
	$$ \kappa''''(E(\boldsymbol{\eta}_\rho) = O(\rho^{2(\beta-4)}) . $$
	Combining~\eqref{decaay} with the previous identity and using that  $E_{ab}(\boldsymbol \eta_\rho) = O(1)$ as $\rho \to 0 $, we obtain
	$$ G_{1234}(\boldsymbol{\eta}_\rho)= H_2(\boldsymbol{\eta}_\rho)  + O ( \rho^{2(\beta-4)+4}) + O(2(\beta-3)+2)  = O(\rho^{2(\beta-2)}) $$
	and hence,
	$$ \lim_{\rho \to 0 } G_{1234}(\boldsymbol{\eta}_\rho)  = H_2(\boldsymbol{\eta}) . $$
	Using the same argument as in the previous step, we have that $G_{234}(\bullet,y,z,k)$ is absolutely continuous on $[-\delta, \delta]$.
\end{itemize}
Since the maps $G_{234}(\bullet,y,k,s)$, $G_{34}(x,\bullet, k,s)$, $G_{4}(x, y,\bullet, s)$, and $G(x,y,z, \bullet)$  are absolutely continuous on $[-\delta, \delta]$, they satisfy the Fundamental Theorem of Calculus, i.e., for every $(x,y,z,s)\in [-\delta, \delta]^4$ we have:
\begin{align*} \int_{0}^x G_{1234}(\alpha,y,z,s) \di \alpha = G_{234}(x,y,z,s)- G_{234}(0,y,z,s),\\
	\int_{0}^y G_{234}(x,\beta,z,s) \di \beta = G_{34}(x,y,z,s)- G_{234}(x,0,z,s),\\
	\int_{0}^z G_{34}(x,y,\gamma,k) \di \gamma = G_{4}(x,y,z,k)- G_{234}(x,y,0, s),\\
	\int_0^k G_4(x,y,z,\mu) \di \mu = G(x,y,z,s) - G(x,y,z,0).
\end{align*}
Therefore,
\begin{align*} A_{hkpq} =& \iiint_{[0,h]\times[0,k] \times [0,p]\times [0,q]} G_{1234}(\boldsymbol{x}) \di \boldsymbol{x} \\
	=&hkpq \iiint_{[0,1]^4} G_{1234}(x e_1 + y e_2 + z e_3 +  se_4) \, \di x \di y \di z \di s.
\end{align*}
The previous computation has shown that the integrand function has a limit as $h,k,p,q\to 0$ and $G_{1234}$  is uniformly bounded, thus, by dominated convergence. Thus,
$$ \lim_{h,k,p,q\to 0 }  \frac{A_{hkpq}}{hkpq} $$
exists finite. By~\Cref{tecnico2}  $D^{ji} f$ is well defined in the $L^2$ sense.

\subsection{Computations on covariances}
In this subsection we compute several covariances involving the field and its derivatives.

\begin{lemma}\label{lem::calcolo1} Let $g:= T \circ  \varphi_N^{-1}$ where $T$ is an isotropic Gaussian random field on $\S^d$ with covariance function $\kappa$ and $\mathrm{CRI}>2$.  Then for every $s\in \R^d$ we have
	\begin{subequations}
		\begin{align}
			\label{calcolo1uguale}
			&\E[ g(t) D^{i}g(s)] = 0,\\
			\label{calcolo2uguale}
			&\E[ D^q g(s) D^i g(s)]  = \frac{4\delta_{iq}}{(1+\nn s)^2} \kappa'(1), \\
			\label{calcolo3uguale}
			&\E[D^qg(s) D^{ji} g(s)] =  \frac{8( s_q \delta_{ij} - s_i \delta_{jq}- s_j \delta_{iq})}{(1+\nn s)^3} \kappa'(1),\\
			\nonumber
			&\E[D^{pq}g(s) D^{ji} g(s)] =  16\frac{\delta_{i j} \delta_{p q} }{(1+\nn s)^3} \kappa'(1) \\
			& 			\label{calcolo4uguale} + \frac{- 2 \delta_{i j} {s}_{p} {s}_{q} + \delta_{i p} s_j {s}_{q} + \delta_{i q} s_j {s}_{p} + \delta_{j p} s_i {s}_{q} + \delta_{j q} s_i {s}_{p} - 2 \delta_{p q} s_i s_j }{(1+\nn{s})^{4}} \kappa'(1)\\
			\nonumber	&+ 16 \frac{ \delta_{i j} \delta_{p q} + \delta_{i p} \delta_{j q} + \delta_{i q} \delta_{j p}}{(1+\nn s)^4},\\
			\label{calcolo5uguale}
			&\E[ g(s) D^{ji}g(s) ]  = - \frac{4\delta_{ij}}{(1+\nn s)^2}
		\end{align}
	\end{subequations}
	where $s_k$ denotes the $k$-th component of the vector $s$.
\end{lemma}
\begin{lemma}\label{lemm::calcolo2} Let $g$ be as in the previous lemma and let $t,s\in \R^d$.	If $t\neq s$ and $ t\neq -\frac{1}{\nn t} t $,  then
	\begin{subequations}
		\begin{align}
			\label{calcolo1}
			\E[ g(t) D^{i}g(s)] &= \frac{L^i(t,s)}{(1+\nn t) (1+\nn s)^2 }  \kappa^{(1)}\tonde{x_{t,s}}, \\
			\label{calcolo2}
			\E[ D^q g(t) D^i g(s)]  &= \sum_{\ell=1}^2 \frac{\beth^\ell _{qi}(t,s)}{(1+\nn t)^{\ell+1} (1+\nn s)^{\ell+1}} \kappa^{(\ell)}(x_{t,s}),\\
			\label{calcolo3}
			\E[D^qg(t) D^{ji} g(s)] &= \sum_{\ell=1}^3 \frac{\gimel^{\ell}_{qij}(t,s)}{(1+\nn{t})^{\ell+1} (1+\nn s)^{\ell+2}} \kappa^{(\ell)}(x_{t,s}),  \\
			\label{calcolo4}
			\E[D^{pq}g(t) D^{ji} g(s)] &=\sum_{\ell=1}^4 \frac{\daleth^\ell_{pqij}(t,s)}{(1+\nn{t})^{\ell+2}(1+\nn s)^{\ell+2}} \kappa^{(\ell)}(x_{t,s}), \\
			\label{calcolo5}
			\E[ g(t) D^{ji}g(s) ]  &= \sum_{\ell=1}^2 \frac{ C^\ell_{ij}(t,s)}{(1+\nn t)^\ell (1+\nn s)^{\ell+2}}\kappa^{(\ell)}(x_{t,s}).
		\end{align}
	\end{subequations}
	where $\kappa^{(\ell)}$ denotes the $\ell$-th derivative of $\kappa$, and the remaining functions are defined in Appendix~\ref{appB}.
\end{lemma}
Let us prove the two lemmas.
\begin{proof} Let $ f:\R \to \R$ be a function such that
	$$ f(t+he_k,s)  \dot{=} f(t,s) + h \widetilde f(t,s)  $$
	where $\dot{=}$ denotes equality to first order in $h$. Then, for every $\alpha, \beta\in \R$ and $\gamma\in \s{ 1, \dots, d}$ we have
	\begin{align*}  &\Bigg[ \frac{f(\zeta,s) g(x_{\zeta,s})}{(1+\nn {\zeta})^\alpha (1+\nn s)^\beta} \Bigg]_{\zeta = t}^{\zeta = t + he_\gamma} 	\dot{=}\frac{1}{(1+\nn {t+he_\gamma})^\alpha (1+\nn t)^\alpha (1+\nn s)^\beta} \\
		&\quad \cdot \Bigg( \Big(f(t,s) + h\widetilde f(t,s) \Big) g(x_{t+he_\gamma,s}) (1+\nn t)^\alpha
		- f(t,s)g(x_{t,s}) (1+\nn{t+he_\gamma})^\alpha \Bigg)
	\end{align*}
	where for every  function $h:\, \R^2\to \R$,  $[h(\zeta,s)]_{\zeta=x_1}^{\zeta=x_2} $ denotes the difference $h(x_2,s)- h(x_1, s)$.
	Now, we observe that  $$(1+\nn{t+he_\gamma })^\alpha  \dot{=} (1+\nn t)^\alpha +  2\alpha t_\gamma (1+\nn t)^{\alpha-1}$$ and so
	\begin{align*} \Bigg[ \frac{f(\zeta,s) g(x_{\zeta,s})}{(1+\nn {\zeta})^\alpha (1+\nn s)^\beta} \Bigg]_{\zeta = t}^{\zeta = t + he_\gamma}
		\dot{=}\frac{f(t,s)}{(1+\nn {t+he_\gamma})^\alpha  (1+\nn s)^\beta} \Bigg( g(x_{t+he_\gamma ,s})- g(x_{t,s})\Bigg)\\
		+h \frac{\widetilde f(t,s) g(x_{t+he_\gamma,s}) (1+\nn t) - 2\alpha t_\gamma f(t,s) g(x_{t,s}) }{(1+\nn {t+he_\gamma})^\alpha (1+\nn t) (1+\nn s)^\beta}.
	\end{align*}
	In particular,  if $g'$ is continuous in the closed interval with endpoints $x_{t+he_\gamma ,s}$ and $x_{t,s}$, by the mean value theorem we have
	\begin{equation}\label{limt+tneqs}
		\begin{aligned}\lim_{h\to 0 } \frac{1}{h}   \Bigg[ \frac{f(\zeta,s) g(x_{\zeta,s})}{(1+\nn {\zeta })^\alpha (1+\nn s)^\beta} \Bigg]_{\zeta = t}^{\zeta = t + he_\gamma } 	=\frac{f(t,s) g'(x_{\zeta,s}) L^\gamma(s,t)}{(1+\nn {t})^{\alpha+2}  (1+\nn s)^{\beta+1}}\\
			+\frac{\tilde f(t,s) (1+\nn t) - 2\alpha t_\gamma f(t,s)  }{(1+\nn {t})^{\alpha+1} (1+\nn s)^\beta}  g(x_{t,s}),
		\end{aligned}
	\end{equation}
	indeed,
	$$ x_{t+he_\gamma ,s}  - x_{t,s}\dot{=}  h \frac{L^\gamma (s,t)}{(1+\nn t) (1+\nn t+he_\gamma) (1+\nn s)}.$$
	Moreover, if $s=t$, $f(s,s) = 0 $ and $g$ are continuous in $[-1,1]$ then
	\begin{equation}\label{lmitet+t=s}
		\begin{aligned}\lim_{h\to 0 } \frac{1}{h}   \Bigg[ \frac{f(\zeta,s) g(x_{\zeta,s})}{(1+\nn {\zeta})^\alpha (1+\nn s)^\beta}\Bigg]_{\zeta = s}^{\zeta = s + he_\gamma} 	=
			\frac{\tilde f(s,s)  }{(1+\nn {s})^{\alpha+\beta}}  g(1).
		\end{aligned}
	\end{equation}
	One can derive a similar result if $f(t,s+he_\gamma) = f(s,t) + h \widehat{f}(t,s)$.  More precisely,  if $g'$ is continuous in the closed interval with endpoints $x_{t+he_\gamma,s}$ and $x_{t,s}$, then
	\begin{equation}\label{lims+tneqs}
		\begin{aligned}
			\lim_{h\to 0 } \frac{1}{h}   \Bigg[ \frac{f(t,\zeta ) g(x_{t,s})} { (1+\nn t)^\alpha (1+\nn {\zeta})^\beta } \Bigg]_{\zeta = s}^{\zeta = s + he_\gamma}
			=\frac{f(t,s) g'(x_{t,s}) L^\gamma(t,s)}{(1+\nn {t})^{\alpha+1}  (1+\nn s)^{\beta+2}}\\
			+\frac{\widehat f(t,s) (1+\nn s) - 2\alpha s_\gamma f(t,s)  }{(1+\nn {s})^{\beta+1} (1+\nn t)^\alpha}  g(x_{t,s}).
		\end{aligned}
	\end{equation}
	Moreover, if $s=t$, $f(s,s) =0$ and $g$ are continuous in $[-1,1]$ then
	\begin{equation}
		\label{limites+t=s}
		\begin{aligned}\lim_{h\to 0 } \frac{1}{h}   \Bigg[ \frac{f(\zeta,s) g(x_{\zeta,s})}{(1+\nn s)^\alpha(1+\nn {\zeta})^\beta }\Bigg]_{\zeta = s}^{\zeta = s + he_\gamma}
			=
			\frac{\widehat f(s,s)  }{(1+\nn {s})^{\alpha+\beta}}  g(1).
		\end{aligned}
	\end{equation}
	In the following, we shall employ several identities involving the function appearing in the Appendix~\ref{appB}. These can be verified through straightforward computations. For convenience, a program enabling such verifications is available at~\url{https://github.com/simmaco99/CriticalPointRandomNeuralNetworks}.
	\begin{itemize}
		\item[\eqref{calcolo1}] Using~\Cref{tecnico2}, we have
		\begin{align*} \E[ g(t) D^{i}g(s)]   = &\lim_{h\to 0 } \frac{1}{h} \Big(\E[ g(t) g(s+he_i) ] - \E[ g(t)g(s)]\Big) = \lim_{h\to 0 }  \frac 1 h \Bigg[  \kappa(x_{t,\zeta})\Bigg]_{\zeta = s}^{\zeta = s+he_i} .
		\end{align*}
		Since $\kappa\in C^2([-1,1])$, one can use~\eqref{lims+tneqs} with $f(t,s)=1$, $g(t,s) = \kappa(x_{t,s})$,  $\alpha=\beta = 1$ and $\gamma=i$. Since $\widehat f(t,s) = 0$ the claim follows. Moreover, since $L^i(s,s) =0 $, we obtain also~\eqref{calcolo1uguale}.
		\item[\eqref{calcolo2}] Using the same argument as in the previous step and~\eqref{calcolo1}, we have
		\begin{align*} \E[ D^q g(t) D^{i}g(s)]   = \lim_{h\to 0 }  \frac 1 h \Bigg[  \frac{L^i(\zeta,s)}{(1+\nn \zeta) (1+\nn s)^2 }\kappa'(x_{\zeta,s})\Bigg]_{\zeta = t}^{\zeta = t+he_p} .
		\end{align*}
		One can prove that
		$$ L^i(t+he_q, s) \dot{=} L^i(t,s) + h\widetilde L(t,s)$$
		and hence using~\eqref{limt+tneqs} with $\alpha = 1 $, $\beta=2$, $f (t,s) =  L^i(t,s)$,  $g= \kappa'(x_{t,s})$ and $\gamma=p$ we obtained  the  desired identity. On the other hand, since $$\beth_{qi}^1 (s,s) = 4 ( 1+ \nn s) \delta_{iq}, \qquad \beth_{qi}^2(s,s) =0 $$ we have~\eqref{calcolo2uguale}.
		\item[~\eqref{calcolo3}] As in the previous step, using~\eqref{calcolo2}  we have
		\begin{align*} \E[ D^q g(t) D^{ji}g(s)]  =\sum_{\ell=1}^2 \Bigg[ \frac{\beth^\ell _{qi}(t,\zeta)}{(1+\nn t)^{\ell+1} (1+\nn \zeta)^{\ell+1}} \kappa^{(\ell)}(x_{t,\zeta}) \Bigg]_{\zeta=s}^{\zeta = s+ h e_j} .
		\end{align*}
		Now, one can prove that
		$$ \beth^\ell _{qi}(t,s+he_j)  \dot{ =}  \beth^\ell _{qi}(t,s)  + \widehat{\beth^\ell _{qij}}(t,s) , \qquad \ell = 1,2.$$
		If $x_{t,s}\neq 1$ then we can apply~\eqref{lims+tneqs} and hence
		\begin{align*} \E[ D^q g(t) D^{ji}g(s)]  =&\sum_{\ell=1}^2 \Bigg[ \frac{\beth_{qi}^\ell(t,s)L^j(t,s)  \kappa^{(\ell+1)}(x_{t,s})  }{(1+\nn {t})^{\ell+2}  (1+\nn s)^{\ell+3}}  \\
			&+\frac{\widehat{\beth_{qij}^\ell} (t,s) (1+\nn s) - 2(\ell+1) s_j \beth_{qi}^\ell(t,s) }{(1+\nn t)^{\ell+1} (1+\nn {s})^{\ell+2} }  \kappa^{(\ell)}(x_{t,s}))\Bigg] \\
			=&\sum_{\ell=2}^3  \frac{\beth_{qi}^{\ell-1}(t,s) L^j(t,s)  \kappa^{(\ell)}(x_{t,s})  }{(1+\nn {t})^{\ell+1}  (1+\nn s)^{\ell+2}} \\
			& +\sum_{\ell=1}^2\frac{\widehat{\beth_{qij}^\ell} (t,s) (1+\nn s) - 2(\ell+1) s_j \beth_{qi}^\ell(t,s) }{(1+\nn t)^{\ell+1} (1+\nn {s})^{\ell+2} }  \kappa^{(\ell)}(x_{t,s}))
		\end{align*}
		This yields the desired identity in a straightforward way.  We note that  $	\beth_{qi}^2(s,s) =0$, so one can apply~\eqref{limites+t=s} instead~\eqref{lims+tneqs}. After  computing  $\beth_{qi}^\ell (s,s)$ and $\widehat{\beth_{qi}^\ell }(s,s)$, we obtain~\eqref{calcolo3uguale}.
		\item[~\eqref{calcolo4}]As above, using~\eqref{calcolo3} we obtain
		\begin{align*} \E[ D^{pq} g(t) D^{ji}g(s)]  =\sum_{\ell=1}^3 \Bigg[ \frac{\gimel^\ell _{qij}(\zeta,s)}{(1+\nn \zeta )^{\ell+1} (1+\nn s)^{\ell+2}} \kappa^{(\ell)}(x_{t,\zeta}) \Bigg]_{\zeta=s}^{\zeta = s+ h e_j} .
		\end{align*}
		Now, one can check that
		$$ \gimel^\ell _{qij}(t+he_p,s)  \dot{ =}  \gimel^\ell _{qij}(t,s)  + \widetilde{\gimel^\ell _{qpij}}(t,s) , \qquad \ell = 1,2,3  . $$
		Since $\kappa\in C^\infty((-1,1)$ then, for $t\neq s $ and $t \neq - \frac{1}{\nn  s}  s$,
		\begin{align*} \E[ D^{pq} g(t) D^{ji}g(s)]  =&\sum_{\ell=1}^3 \Bigg[ \frac{\gimel^\ell _{qij}(t,s) L^p(t,s)}{(1+\nn t )^{\ell+3} (1+\nn s)^{\ell+3}} \kappa^{(\ell+1)}(x_{t,s}) \\
			&+  \frac{\widetilde{\gimel^\ell_{qpij}}(t,s) (1+\nn t) - 2(\ell+1) t_p \gimel^\ell_{qij }}{(1+\nn {t})^{\ell+2} (1+\nn s)^{\ell+2}}  \kappa^{(\ell)}(x_{t,s}) \Bigg]\\
			=&\sum_{\ell=2}^4 \frac{\gimel^{\ell-1} _{qij}(t,s) L^p(t,s)}{(1+\nn t )^{\ell+2} (1+\nn s)^{\ell+2}} \kappa^{(\ell)}(x_{t,s}) \\
			&+\sum_{\ell=1}^3  \frac{\widetilde{\gimel^\ell_{qpij}}(t,s) (1+\nn t) - 2(\ell+1) t_p \gimel^\ell_{qij }}{(1+\nn {t})^{\ell+2} (1+\nn s)^{\ell+2}}  \kappa^{(\ell)}(x_{t,s})\\
			= & \frac{\widetilde{\gimel^1_{qpij}}(t,s) (1+\nn t) - 4 t_p \gimel^1_{qij }}{(1+\nn {t})^{3} (1+\nn s)^{3}}  \kappa^{(1)}(x_{t,s}) +  \frac{\gimel^{3} _{qij}(t,s) L^p(t,s) \kappa^{(4)}(x_{t,s}) }{(1+\nn t )^{6} (1+\nn s)^{6}} \\
			& + \sum_{\ell=2}^3 \frac{\gimel^{\ell-1} _{qij}(t,s) L^p(t,s)+ \widetilde{\gimel^\ell_{qpij}}(t,s) (1+\nn t) - 2(\ell+1) t_p \gimel^\ell_{qij }}{(1+\nn {t})^{\ell+2} (1+\nn s)^{\ell+2}}  \kappa^{(\ell)}(x_{t,s}) .
		\end{align*}
		Moreover, $\gimel_{qij}^2(s,s) = \gimel_{qij}^3(s,s) =0 $. After the computation of $\gimel_{qpij} ^1(s,s)$, $\widetilde{\gimel_{qpij} ^\ell}(s,s)$ for $\ell = 1,2,3$ we have~\eqref{calcolo4uguale}.
		\item[~\eqref{calcolo5}] Using~\Cref{tecnico2} and~\eqref{calcolo1} we have
		\begin{align*}
			\E[ g(t) D^{ji} g(s) ] = \lim_{h \to 0 } \Bigg[  \frac{L^i(t,\zeta)}{(1+\nn t) (1+\nn \zeta)^2}\Bigg]_{\zeta = s}^{\zeta=s+he_j} .
		\end{align*}
		The claim follows from~\eqref{lims+tneqs}.
	\end{itemize}
\end{proof}

\subsection{Proof of~\ref{it:ii}:  seconds derivatives exist almost surely}\label{B2}
The triple $(D^i f(s), D^i f(t), D^if(s))$ is a (multivariate) Gaussian random vector. Indeed, for all $a_1,a_2, a_3\in \R$, the linear combination $ a_1 D^i f(s) + a_2 D^i f(t)+ a_3 D^{ji} f(s) $ is the $L^2$ limit of Gaussian random variables.  Set
$$ \sigma_n^2 :=  \Var(n (D^i f(s+e_j /n) - f(s)) -D^{ji}f(s)) , $$
then the random variable $n (D^if(s+e_i/n) - D^if(s)) -D^{ji} f(s)$ has the same distribution as $\sigma_n Z$, where $Z\sim \mathcal N(0,1)$. Hence,
\begin{align*}\label{bc}
	\P\tonde{| n (D^i f(s+e_i/n) - D^i f(s)) -D^{ji} f(s)|>\varepsilon
	} \leq  2 \P\tonde{Z> \frac{\varepsilon}{\sigma_n}}  \leq \frac{2\sigma_n}{\sqrt{2\pi}\varepsilon} \mathrm e^{-\frac{\varepsilon^2}{2\sigma_n^2}}  .
\end{align*}
If we show that $ \sigma_n ^2 = O(n^{-\zeta}) $ for some $\zeta> 0$, then~\eqref{somma} follows.  For notational convenience, we set  $\frac{1}{n} = h$. Using~\Cref{lem::calcolo1} and~\Cref{lemm::calcolo2}
\begin{align*}
	h^2	\sigma_{1/h}^2 =& \E\Big[ (D^i f(s+ he_j) - D^i f(s))^2 \Big] + h^2\E[D^{ji}f(s)^2]\\
	&- 2h \E[ D^{ji}f(s) (D^i f(s+ he_j) - D^i f(s))]\\
	=&  \kappa'(1)\Bigg( \frac{4}{(1+\nn{s+he_j})^2} + \frac{4}{(1+\nn s)^2}  -\frac{16hs_j}{(1+\nn s)^3}\\
	&\quad +\frac{h^2((1+\nn s)\delta_{ij} - 2\delta_{ij}s_i s_j + s_i^2 + s_j^2)}{(1+\nn s)^4} \Bigg)  + \kappa''(1) \frac{16h^2 ( 1 + 2\delta_{ij})}{(1+\nn s)^4}\\
	& -2\kappa'(x_{t,s}) \Bigg( \frac{ \beth_{ii}^1(s+he_j,s)}{(1+\nn s)^{2}(1+\nn{s+he_j})^2} + \frac{h\gimel_{iij}^1(s+he_j,s)}{(1+\nn{s+he_j})^2 (1+\nn s)^3}\Bigg) \\
	& -2 \kappa''(x_{t,s})\Bigg( \frac{ \beth_{ii}^2(s+he_j,s)}{(1+\nn s)^{3}(1+\nn{s+he_j})^3} + \frac{h\gimel_{iij}^2(s+he_j,s)}{(1+\nn{s+he_j})^3 (1+\nn s)^4}\Bigg)\\
	& - 2 \kappa'''(x_{t,s}) \frac{h\gimel_{iij}^3(s+he_j,s)}{(1+\nn{s+he_j})^4 (1+\nn s)^5} .
\end{align*}
Using the definition of CRI and some algebraic manipulations,
	\begin{align*}
		h^2 \sigma_{1/h} =  &  \frac{2\kappa'(1) r_1(h)}{(1+\nn s)^4 (1+\nn{s+he_j})^2}+ \frac{2\kappa''(1) r_2(h)}{(1+\nn s)^4(1+\nn{s+he_j})^3} \\
		& \quad +\frac{ c\beta 2^{\beta-2} h^{2\beta-5} r_\beta(h)}{(1+\nn s)^{\beta+2} (1+\nn{s+he_j})^{\beta+1}}  + o( r(h;\beta))  .
	\end{align*}
	where $r_1(h), r_2(h), r_\beta(h)$, and $r(h;\beta)$ are given in Appendix~\ref{appB2}.
	In the same appendix, we also prove that,  as $h \to 0 $,  the following asymptotics hold: \begin{align*}
		r_1(h) = R_1 h^4 + O(h^5),\\
		r_2(h) = R_2 h^4 + O(h^5),\\
		r_3 (h) = R_\beta h^3 + O(h^4),\\
		r(h) = O(h^{2(\beta-2)})
	\end{align*}
	where $R_k\neq 0$ is independent of $h$.
	Define
	$$ R(\beta) = \frac{2\kappa'(1) (1+\nn{s}) R_1 + 2 \kappa''(1) R_2}{(1+\nn s)^7} \mathbb I_{[3, \infty)} (\beta ) +\frac{ c\beta 2^{\beta-2} R_\beta }{(1+\nn s)^{2\beta+3} }  \mathbb I_{(2,3]} (\beta)  , $$
	and
	$$ \rho(\beta) = 1 \wedge (\beta-2)$$
	we obtain that
	$$ \sigma_{1/n}^2 = R(\beta) n^{-2\rho(\beta)}  + o(n^{-2\rho(\beta)})  . $$
	Recalling that  $\beta>2$, thus setting$\zeta = 2\rho(\beta)$ the claim follows.
	\subsection{Proof of~\ref{it:iii}: seconds derivatives are H\"older continuous}\label{A4}
	Since $\Big(D^{ji} g(t), D^{ji}g(s)\Big) $ is a  multivariate Gaussian random vector,
	$$ \E\quadre{\abs{  D^{ji} g(t) - D^{ji}g(s)}^{2k}} =  \E\quadre{|   D^{ji} g(t) - D^{ji}g(s)|^2}^k (2k-1)!!.$$
	It suffices to prove that
	$$ \E\quadre{|D^{ji} g(t) - D^{ji}g(s)|^2} \leq C \n{t-s}^{2\rho(\beta)} . $$
	Now, using~\Cref{lem::calcolo1} and~\Cref{lemm::calcolo2}, the definition of $\mathrm{CRI}$ and some simple algebraic manipulations, we have
	\begin{align*}
		\E\quadre{\abs{D^{ji} g(t) - D^{ji}g(s)}^2}=&
		2\frac{\kappa'(1) H_1 (t,s)   + \kappa''(1) H_2(t,s)}{(1+\nn s)^4(1+\nn t )^4} -2 \frac{H_3 (t,s,\beta)}{(1+\nn t)^{\beta+2} (1+\nn s)^{\beta+2}} \\
		& \quad + o \tonde{  \frac{H_3 (t,s,\beta)}{(1+\nn t)^{\beta+2} (1+\nn s)^{\beta+2}} }  \\
	\end{align*}
	where $H_k $, $k=1, \dots, 3$ are given in Appendix~\ref{appB3}. In the same appendix, we prove that, as $\nn{t-s} \to 0 $ then
	$$H_k(s,t) \leq C_k (1+\nn s)^4 (1+\nn t )^4\nn{t-s} ( 1 + o(1))\qquad k =1,2$$
	and
	$$ H_3(s,t) \leq C_3( 1 +\nn s)^4 ( 1 + \nn t)^4 \n{t-s}^{2(\beta-2)} ( 1 + o(1)) $$
	where $C_k$ are universal constants. Thus the claim follows.
	\section{Generalization of~\cite[Lemma 4.1]{cheng2018expected} }\label{riemann}
	In this section, we show that~\cite[Lemma 4.1]{cheng2018expected}  continues to hold  under the assumption $\mathrm{CRI}>2$.

	Let $(E_1,\dots,E_d)$ denote an orthonormal frame on $\S^d$, i.e., a family of smooth vector fields forming an orthonormal basis of each tangent space with respect to the Riemannian metric. For a smooth function $f$ on $\S^d$, we write $E_i f$ to denote the derivative of $f$ in the direction of $E_i$.

	The second derivatives are defined via the Levi-Civita connection $\nabla$. Namely, the Hessian of $f$ is given by
	\[
	E_{ij} f:= E_i(E_j f) - \nabla_{E_i} E_j f,
	\]
	where
	\[
	\nabla_{E_i} E_j:= \sum_{k=1}^d \theta_{ij}^k E_k,
	\qquad
	\theta_{ij}^k:= g(\nabla_{E_i} E_j, E_k),
	\]
	and $g$ denotes the Riemannian metric on $\S^d$.

	\begin{lemma} \label{lemma_gen}
		Let  $T$ be an isotropic Gaussian random field on $\S^d$ with $\mathrm{CRI}>2$ and covariance function $\kappa$. Then for every point $p\in \S^d$,
		\begin{subequations}
			\begin{gather}
				\label{p1}\E[  E_i T(p) T(p)] = 0,\\
				\label{p2}\E[ E_i T(p) E_j T(p) ] = \kappa'(1) \delta_{ij},\\
				\label{p3} \E[ E_{ij} T(p) T(p)]  = - \kappa'(1) \delta_{ij},\\
				\label{p4} \E[E_{ij}T(p) E_{q} T(p) ] = 0,\\
				\label{p5} \E[E_{ij}T(p) E_{kq}T(p)] = \kappa'(1) \delta_{ij}\delta_{kq} + \kappa''(1) \Big( \delta_{ij}\delta_{kq} + \delta_{ik } \delta_{jq} + \delta_{iq} \delta_{jk}\Big).
			\end{gather}
		\end{subequations}
	\end{lemma}
	\begin{proof} Let $\varphi_N:\S^d \setminus \{N\}\to\R^d$ denote the stereographic projection from the north pole, one can construct a natural orthonormal frame $(E_1,\dots,E_d)$ by setting
		\[
		(E_i T)(p):= \Omega(\varphi_N(p)) \, D^i f_N(\varphi_N(p)),
		\]
		where $f:= T \circ \varphi_N^{-1}$,
		\[
		D^i f_N(s):= \left.\frac{\partial f_N}{\partial x_i}(x)\right|_{x=s},
		\qquad
		\Omega(x):= \frac{1+\|x\|^2}{2}.
		\]
		Using Koszul's formula and the fact that the frame is orthonormal, one obtains
		\[
		2\theta_{ij}^k = g(E_k,[E_i,E_j]) - g(E_j,[E_k,E_i]) - g(E_i,[E_k,E_j]),
		\]
		where the Lie bracket is defined by
		\[
		[X,Y]T:= X(YT) - Y(XT).
		\]
		A straightforward computation shows that
		\[
		[E_i,E_j] T(p) = \varphi_N(p)_i \, E_j - \varphi_N(p)_j \, E_i,
		\]
		where $\varphi_N(p)_j$ denotes the $j$-th component of the vector $\varphi_N(p)$. Hence,
		\begin{align}\label{def_theta}
			\theta_{ij}^k(x) = x_k \delta_{ij} - x_j \delta_{ik}.
		\end{align}
		In conclusion, we have the explicit expression for the Hessian:
		\begin{equation}\label{hess}
			E_{ij} T(p) = E_i(E_j T)(p) - \sum_{k=1}^d \theta_{ij}^k(\varphi_N(p)) \, E_k T(p).
		\end{equation}

		To simplify the notation, let $t = \varphi_N(p)$.
		\begin{itemize}
			\item  [\eqref{p1}]
			Since $$ \E[ E_i T (p) T(p)] = \Omega(t) \E[ D^i f_N(t) f_N(t)]   , $$
			using~\eqref{calcolo1uguale} we obtain the claim.
			\item[\eqref{p2}] Since
			$$ \E[ E_i T(p) E_j T(p) ]   = \Omega(t)^2 \E[ D^i f_N(t)  D^j  f_N(t)] $$
			the claim follows by~\eqref{calcolo2uguale}.
			\item [\eqref{p3}]  Using the explicit expression for the Hessian~\eqref{hess},  we obtain
			\begin{align*}
				\E[E_{ij}T(p) T(p)]	 = \E[ E_i(E_j T)(p) T(p)]- \sum_{k=1}^d \theta_{ij}^k  \E[ E_k T(p) T(p)].
			\end{align*}
			Now, using~\eqref{p1}, we can rewrite the previous identity in this way
			$$\E[E_{ij}T(p) T(p)]	 = \E[ E_i(E_j T)(p) T(p)].$$
			Since,
			\begin{equation}\label{hess2}
				\begin{aligned} E_i (E_j T)(p)  = &\Omega(t) \frac{\partial}{\partial t_i} \Big( \Omega(t)  \frac{\partial}{\partial t_j} f_N(t)\Big)  \\
					= &\Omega(t)^2 \frac{\partial^2}{\partial t_i \partial t_j} f_N(t) + t_i \Omega(t)  \frac{\partial}{\partial t_j} f_N(t) \\
					=&   \Omega(t)^2 \frac{\partial^2}{\partial t_i \partial t_j} f_N(t) + t_i  E_j T(p)
				\end{aligned}
			\end{equation}
			we obtain
			\begin{align*} \E[ E_{ij} T(p) T(p)]   = & \Omega(t)^2 \E\Big[D^{ij}f_N(t) f(t) \Big]
				+ t_i \E[ E_j T(p) T(p) ] .
			\end{align*}
			The claim following from~\eqref{calcolo5uguale} and~\eqref{p1}.

			\item[\eqref{p4}]Combining~\eqref{hess}  with~\eqref{hess2} we have,
			\begin{align*}
				\E[ E_{ij} T(p) E_q T(p)]   =  & \Omega(t)^3 \E\Big[D^{ij}f_N(p)D^q f(p) \Big]   + t_i \E[ E_j T(p)E_q T(p)]  \\
				&- \sum_{k=1}^d \theta_{ij}^k (t) \E[ E_k T(p) E_q T(p)]  .
			\end{align*}
			Using~\eqref{calcolo2uguale} and~\eqref{p2}  we obtain
			\begin{align*}
				\E[ E_{ij} T(p) E_q T(p)]   =  & \kappa'(1) \Bigg[t_q \delta_{ij} -  t_i   \delta_{qj}  - t_j \delta_{iq} + t_i \delta_{jq}  -\theta_{ij}^q (t)\Bigg]  .
			\end{align*}
			The claim follows from the definition of $\theta_{ij}^k $ (cf.,~\eqref{def_theta}).
			\item[\eqref{p5}]Using~\eqref{hess} we obtain
			\begin{align*}&\E[ E_{ij} T(p) E_{kq} T(p)] =\E[ E_i(E_j T)(p)] \\
				& \quad -\sum_{r=1}^d \Big( \theta_{ij}^r (s)\E[E_r T(p) E_k(E_q T)(p) ] +  \theta_{kq}^r(s) \E[E_r T(p) E_{i}(E_j T)(p)] \Big) \\
				& \quad + \sum_{r,\ell =1}^d  \theta_{ij}^r(s)\theta_{kq}^\ell(s) \delta_{r\ell} .
			\end{align*}
			Using~\Cref{lem::calcolo1} and the previous computations, one can prove that
			$$  \E[E_r T(p) E_{i}(E_j T)(p)] = \theta_{ij}^r(s) $$
			and
			\begin{align*}\E[E_i(E_jT)(p)E_p(E_q T)(p)] &= \kappa'(1)\Big( 1 + \nn{s}) \delta_{i j} \delta_{p q} - \delta_{i j} {s}_{p} {s}_{q} + \delta_{i p} {s}_{j} {s}_{q} - \delta_{p q} {s}_{i} {s}_{j}\Big)\\
				&+
				\kappa''(1) \Big(\delta_{i j} \delta_{p q} + \delta_{i p} \delta_{j q} + \delta_{i q} \delta_{j p}\Big) .
			\end{align*}
			The claim follows,nothing that
			$$ \sum_{r=1}^d \theta_{ij}^r \theta_{pq}^r = \nn s  \delta_{i j} \delta_{p q} - \delta_{i j} {s}_{p} {s}_{q} + \delta_{i p} {s}_{j} {s}_{q} - \delta_{p q} {s}_{i} {s}_{j} .   $$
		\end{itemize}
	\end{proof}
	\section{Additional details}\label{appC}
	In this section, we provide further details on the proof in Appendix~\ref{appA}.We also explain how to use the symbolic computation program (publicly available at~\url{https://github.com/simmaco99/CriticalPointRandomNeuralNetworks}).Throughout this appendix, we use the following abbreviations:
	$$ S:= 1 + \nn s, \qquad T : = 1 + \nn t , \qquad D := \n{t-s} . $$
	Moreover, for each $k\in\{1,\dots,d\}$, we set $d_k:=t_k-s_k$.
	\newpage
	\subsection{Details for the proof of~\Cref{lemm::calcolo2}}
	\label{appB}
	\begin{align*}
		L^i(t,s):=&4  D^2 s_i + 4  Sd_i, \\
		\beth_{qi}^{1}(t,s):=&\widetilde{L_q^i}(t,s)  T- 2  t_q L_i(t,s) , \\
		\beth_{qi}^{2}(t,s):=&L^i(t,s)  L^q(s,t) , \\
		\gimel_{qij}^{1}(t,s):=&\widehat{\beth_{qij}^{1}}(t,s)S- 4  s_j \beth_{qi}^{1}(t,s) , \\
		\gimel_{qij}^{2}(t,s):=&\beth_{qi}^{1}(t,s)  L_j(t,s) + \widehat{\beth_{qij}^{2}}(t,s) S - 6s_j  \beth_{qi}^{2}(t,s) , \\
		\gimel_{qij}^{3}(t,s):=&\beth_{qi}^{2}(t,s)  L_j(t,s) ,  \\
		\daleth_{qpij}^1(t,s):=& \widetilde{\gimel^1_{qpij}}(t,s) (1+\nn t) - 4 t_p \gimel^1_{qij }(t,s) , \\
		\daleth_{qpij}^\ell(t,s) :=& \gimel^{\ell-1} _{qij}(t,s) L^p(s,t)+ \widetilde{\gimel^\ell_{qpij}}(t,s) T- 2(\ell+1) t_p \gimel^\ell_{qij }(t,s) , \quad \ell = 2,3 ,   \\
		\daleth_{qpij}^4(t,s) :=&\gimel^{3} _{qij}(t,s) L^p(s,t),  \\
		C_{ij}^1 (t,s):= & \widetilde{L_j^i}(t,s) S- 4 L^i(t,s) ,\\
		C_{ij}^2(t,s):=& L^i(t,s) L^j(t,s) , \\
		\widetilde{L_q^i}(t,s) :=& 4S  \delta_{iq}  +8s_i d_q , \\
		\widehat{L_q^i}(t,s) := &4  (-2s_i d_q+ \delta_{iq}D^2+ 2s_q d_i - \delta_{iq}S) , \\
		\widehat{\beth_{qij}^{1}}(t,s) :=&8 T  (d_q \delta_{ij}-s_i  \delta_{jq}  + s_j\delta_{iq}) - 2t_q \widehat{L_j^i}(t,s) , \\
		\widehat{\beth_{qij}^{2}}(t,s) :=&L_i(t,s)  \widetilde{L_j^q}(s,t) + \widehat{L_j^i}(t,s)  L_q(s,t) , \\
		\widetilde{\gimel_{qpij}^{1}}(t,s):= &SM_{qpij}^{1}(t,s) - 4  s_j  N_{qpij}^{1}(t,s) ,  \\
		\widetilde{\gimel_{qpij}^{2}}(t,s):= &N_{qpij}^{1}(t,s)  L_j(t,s) + \widetilde{L_p^j}(t,s)  \beth_{qi}^{1}(t,s) + M_{qpij}^{2}(t,s)  S- 6  s_j  N_{qpij}^{2}(t,s) ,  \\
		\widetilde{\gimel_{qpij}^{3}}(t,s):=  &N_{qpij}^{2}(t,s)  L_j(t,s) + \beth_{qi}^{2}(t,s)  \widetilde{L_p^j}(t,s) , \\
		M_{qpij}^1:=&8\delta_{ij}\delta_{pq} ( S+T-D) + 16  t_q ( \delta_{ij}s_p - \delta_{ip}s_j + \delta_{jp}s_i) ,  \\
		&+ 16\delta_{pq}( s_it_j  - s_jt_i)+16t_p( - \delta_{ij}s_q + \delta_{iq}s_j - \delta_{jq}s_i) ,  \\
		M_{qpij}^2 :=&\widetilde{L_p^i}(t,s)  \widetilde{L_j^q}(s,t) +L_i(t,s)   8 ( -\delta_{jp}t_q + \delta_{jq}t_p \delta_{pq}d_j),   \\
		&+ 8(- \delta_{ij}s_p + \delta_{ij}t_p + \delta_{ip}s_j - \delta_{jp}s_i)  L_q(s,t)+ \widehat{L_j^i}(t,s)  \widehat{L_p^q}(s,t) , \\
		N_{qpij}^1:= &8s_i \delta_{pq}  T + 2  t_p  \widetilde{L_q^i}(t,s) - 2  \delta_{pq}  L_i(t,s) - 2  t_q  \widetilde{L_p^i}(t,s) ,  \\
		N_{qpij}^2:= & L_i(t,s)  \widehat{L_p^q}(s,t) + \widetilde{L_p^i}(t,s)  L_q(s,t) .
	\end{align*}
	To verify the first-order expansions of these functions, one can use the function \texttt{FirstOrderExpansion}. More precisely,
	\begin{itemize}
		\item If \texttt{FirstOrderExpansion(L(i),tildeL(q,i),True,q)} returns \texttt{True} then
		$$ L^i(t+he_q, s) \dot{=} L^i(t,s) + h\widetilde{L^i_q}(t,s)$$
		holds.
		\item If, for all $\ell\in\s{1,}2$,  \texttt{FirstOrderExpansion(beth(q,i,$\ell$),hatbeth(q,i,j,$\ell$),False,j)} returns \texttt{True},  then
		$$ \beth_{qi}^\ell  (t,s+he_j ) \dot = \beth_{qij}^\ell (t,s)  + h \widehat{\beth_{qij}^1}(t,s)  \qquad \ell=1,2.$$
		\item If, for all $\ell \in\s{1,2,3}$\\texttt{FirstOrderExpansion(gimel(q,i,j,$\ell$),tildegimel(q,p,i,j,$\ell$),True,p)} returns \texttt{True},  then
		$$ \gimel _{qij}^\ell  (t+he_p,s ) \dot = \gimel _{qij}^\ell  (t,s) + h \widetilde{\gimel_{qpij}^\ell}(t,s) \qquad \ell=1,2,3.$$

	\end{itemize}

	To compute the value of an expression when $t=s$, we  use the function \texttt{ValueOrigin}.
	\subsection{Details for the proof of~\ref{it:ii}}\label{appB2}
	Define
	$$ \Delta_\ell (h):= \Big( \beth_{ii}^k(s+he_j,s)(1+\nn s)  + h\gimel_{iij}^\ell(s+he_j,s)\Big)\qquad \ell=1,2.$$
	we can then  write
	\begin{align*}
		r_1(h) = &	2S^4+ 2S^2 (1+\nn{s+he_j})^2 + 8h^2 \Big(S\delta_{ij} - 2\delta_{ij}s_i s_j + s_i^2 + s_j^2\Big)(1+\nn{s+he_j})^2 \\
		& -8hs_jS(1+\nn{s+he_j})^2   -S\Delta_1(h),
	\end{align*}
	\begin{align*}
		r_2(h) = & 8h^2 ( 1 + 2\delta_{ij}) (1+\nn{s+he_j})^3 + 2h^2 \Delta_1(h) - \Delta_2(h),
	\end{align*}
	\begin{align*}
		r_\beta(h) =  &-4h^3 \Delta_1(h) + 2(\beta-1)h \Delta_2(h) - (\beta-1)(\beta-2) \gimel_{iij}^3(s+he_j,s) ,
	\end{align*}
	and
	\begin{align*} r(h;\beta) =& h^{2\beta-2} \Delta_1(h)- h^{2\beta-4}\Delta_2(h)+ h^{2\beta-5}\gimel_{iij}^3(s+he_j,s) .
	\end{align*}

	Using the function \texttt{simplify}, we can show that
	\begin{align*}r_1 (h)  = &h^{6} \cdot \left(8 S \delta_{i j} - 16 \delta_{i j} {s}_{i}^{2} + 8 {s}_{i}^{2} + 8 {s}_{j}^{2}\right) + h^{5} \cdot \left(32 S \delta_{i j} {s}_{i} - 8 S {s}_{j} - 64 \delta_{i j} {s}_{i}^{3} + 32 {s}_{i}^{2} {s}_{j} + 32 {s}_{j}^{3}\right)\\
		& + h^{4} \cdot \left(16 S^{2} \delta_{i j} + 2 S^{2} + 16 S {s}_{i}^{2} - 16 S {s}_{j}^{2} - 64 \delta_{i j} {s}_{i}^{4} + 32 {s}_{i}^{2} {s}_{j}^{2} + 32 {s}_{j}^{4}\right) ,
	\end{align*}
	\begin{align*}
		r_2(h) =& h^{8} \cdot \left(16 \delta_{i j} + 8\right) + h^{7} \cdot \left(96 \delta_{i j} {s}_{i} + 48 {s}_{j}\right) + h^{6} \cdot \left(48 S \delta_{i j} + 24 S + 192 \delta_{i j} {s}_{i}^{2} + 96 {s}_{j}^{2}\right) \\
		&+ h^{5} \cdot \left(320 S \delta_{i j} {s}_{i} + 64 S {s}_{j} - 256 \delta_{i j} {s}_{i}^{3} + 192 {s}_{i}^{2} {s}_{j} + 64 {s}_{j}^{3}\right)\\
		& + h^{4} \cdot \left(128 S^{2} \delta_{i j} + 16 S^{2} - 128 S \delta_{i j} {s}_{i}^{2} + 96 S {s}_{i}^{2} + 32 S {s}_{j}^{2}\right),
	\end{align*}
	and
	\begin{align*}
		\Delta_1(h) = 4 S^3 + O (h^3),\\
		\Delta_2(h) = (16S^3\delta_{ij} + 16 S^3) h^2 + O (h^3),\\
		\gimel_{iij}^3(s+he_j,s) = -64 S^3 h^3 \delta_{ij}  + O(h^4) .
	\end{align*}
	Thus,
	$$ r_\beta (h) = 32(\beta-1)(2\beta-3)(1+\delta_{ij}) S^3 h^3+ O(h^4)$$
	and $r(h;\beta) = O(h^{2\beta-2})$.
	\subsection{On the proof of~\ref{it:iii}}\label{appB3}
	\begin{align*} H_1 (s,t) = &8\Big(S\delta_{ij} - 2 \delta_{ij} s_i^2 + s_i^2 + s_j^2\Big)S^4  + 8 \Big(T\delta_{ij} - 2 \delta_{ij} t_i^2 + t_1^2 +  t_j^2\Big)S^4  -\daleth_{ijij}^1 (s,t)TS,
		& \\
		H_2(s,t) = &  8 (1+2\delta_{ij})\Big( T^4 + S^4 \Big)  + 2D \daleth_{ijij}^1 (t,s) - \daleth_{ijij}^2(t,s)\Big), \\
		H_3(s,t) = & \sum_{k=1}^4 2^{\beta-k} \frac{\beta!}{(\beta-k)!}\n{t-s}^{2(\beta-k)} \daleth_{ijij}^k(t,s)  .
	\end{align*}
	Moreover, we obtain
	\begin{align*}H_1(s,t)=& 8\delta_{ij}\Big( ST ( T-S) ( T^2-S^2) + 4 S^2 T^2  d_i^2 - 2  (S-T) ( S^3  t_i^2 - T^3  s_i^2) +2ST  ( St_i - Ts_i)^2 \Big)\\
		&- 64 STd_is_j t_j ( T s_i - S t_i ) -64ST d_j s_i  t_i ( T s_j - S t_j ) \\
		&+ 8( T^2  s_i - S^2  t_i)^2 + 8( T^2  s_j- S^2  t_j)^2\\
		& +8D^2 ST\delta_{ij} \Big(ST - 4 S  {t}_{i}^{2} - 4 T {s}_{i}^{2}\Big)  + 128 D^2 ST16 {s}_{i} {s}_{j} {t}_{i} {t}_{j}.
	\end{align*}
	Hence $$H_1\lesssim \nn{t-s} (1+\nn t)^4 (1+\nn s)^4$$ . Here,  and in what follows, the notation $f\lesssim g$   means that there exists a universal constant $C>0$ such that $f\leq C g $.  This computation also shows that \begin{equation}\label{boundd1}
		\daleth_{ijij}^1(s,t)\lesssim  (1+\nn t)^4 (1+\nn s)^4 \sum_{k\in \s{0, 2} } \n{t-s}^k .
	\end{equation} Moreover,  as $\nn{t-s} \to 0$, we obtain
	\begin{align*}H_2 (s,t) = &16 D^2 \delta_{ij}\Big(S^{2} T - 8 S^{2} {t}_{i}^{2} + S T^{2} + 10 S T {d}_{i} {s}_{i} - 8 S T {d}_{i} {t}_{i} - 2 S T {s}_{i}^{2} + 18 S T {s}_{i} {t}_{i} - 4 S T {t}_{i}^{2} - 8 T^{2} {s}_{i}^{2}
		\Big)\\
		&+ 32D^2 \Big(3 S T {s}_{i} {t}_{i} + 3 S T {s}_{j} {t}_{j} - 20 S {d}_{i} {s}_{j} {t}_{i} {t}_{j} - 20 S {d}_{j} {s}_{i} {t}_{i} {t}_{j} + 20 T {d}_{i} {s}_{i} {s}_{j} {t}_{j} + 14 T {d}_{j} {s}_{i} {s}_{j} {t}_{i}\\
		& \quad  - 6 T {s}_{i} {s}_{j}^{2} {t}_{i} + 6 T {s}_{i} {s}_{j} {t}_{i} {t}_{j}
		\Big)
		\\&+ 8 \Big(
		(S^2-T^2)^2
		+ 8ST d_i  (S t_i -T s_i) +   8ST d_j  (S t_j -T s_j)
		+ 16 STd_id_j s_j t_i
		\\&\quad -16 d_id_j  ( 3T^2 s_is_j+ 3 S^2 t_it_j  -ST(2 s_i  t_j + s_jt_i )) \
		+ ST 32 d_i^2 s_j  t_j +32 STd_j^2s_it_i
		\Big) \\
		& + 16 \delta_{ij} \Big((S^2-T^2)^2 +4 ST d_i^2  (S + T ) +16 STd_i  (St_i - Ts_i)\Big)  +\widetilde{H_2}(t,s)+  2 D^2\daleth_{ijij}^1(t,s) .
	\end{align*}
	Using the function \texttt{BrutalBound} again, we have
	$$\widetilde H_2(t,s) \lesssim (1+\nn t)^3 (1+\nn s)^3 \sum_{k\in \s{ 2,4} } \n{t-s}^k.$$
	By~\eqref{boundd1} we obtain
	$$ H_2 (s,t)\lesssim  (1+\nn t)^4 (1+\nn s)^4 \sum_{k\in \s{ 2,4} } \n{t-s}^k $$
	and
	$$	\daleth_{ijij}^2(s,t)\lesssim   (1+\nn t)^4 (1+\nn s)^4 \sum_{k\in \s{ 0, 2,4} } \n{t-s}^k  . $$
	Using, again, the function \texttt{BrutalBound}, we have that
	$$\gimel^3_{ijij}(t,s) \lesssim (1+\nn t)^3 (1+\nn t)^3  \sum_{k=2}^6 \n{t-s}^k,$$
	$$\gimel^4_{ijij}(t,s) \lesssim (1+\nn t)^2 (1+\nn t)^2 \sum_{k=4}^8 \n{t-s}^k. $$
	Thus
	$$ H_3 \lesssim (1+\nn  t)^4 (1+\nn  s)^4 \sum_{k=0}^4 \n{t-s}^{2\beta -k}. $$

\end{document}